\documentclass[11pt]{article}
\usepackage[T1]{fontenc}
\pdfoutput=1

\usepackage{arydshln}
\usepackage{multicol}
\usepackage{booktabs}
\usepackage{jcd}
\usepackage{conformal_style}
\usepackage{xspace}
\usepackage{overpic}
\usepackage{tikz}
\usepackage{authblk}
\usepackage{subcaption}
\usepackage[font=small, labelfont=bf, margin=1.5em]{caption}
\usemedgeometry
\bluehyperref
\usepackage{natbib}
\usepackage{pgf}
\usepackage{pgfplots}
\usepackage{ifthen}
\usetikzlibrary{shapes.geometric, arrows}

\pgfplotsset{compat=1.11}
\usepgfplotslibrary{fillbetween}
\usetikzlibrary{intersections}
\usetikzlibrary{arrows,automata}
\usetikzlibrary{positioning}

\tikzstyle{startstop} = [rectangle, rounded corners, minimum width=3cm, minimum height=1cm,text centered, draw=black, fill=white!30]
\tikzstyle{arrow} = [thick,->,>=stealth]

\renewenvironment{abstract}{%
    \vspace{0.25em}
    \begin{center}%
        {\bfseries \abstractname \vspace{-0.5em}}%
    \end{center}%
    \quotation \small
}{%
    \endquotation
}

\makeatletter
\renewcommand{\keywords}[1]{%
    \begin{quote}
    \small \textbf{Keywords:} 
    \def\and{\unskip, } 
    #1
    \end{quote}
}
\makeatother

\title{Support Tokens, Stability Margins, and \\ a New Foundation for Robust LLMs \vspace{1em}}

\author[1]{Deepak Agarwal}
\author[1]{Dhyey Dharmendrakumar Mavani}
\author[1]{Suyash Gupta}
\author[1]{Karthik Sethuraman}
\author[1]{Tejas Dharamsi}

\affil[1]{LinkedIn, Mountain View, CA \\ \protect\linebreak \texttt{\{dagarwal, dmavani, suyagupta, ksethuraman, tdharamsi\}@linkedin.com}}

\date{\vspace{-3em}} 

\begin{document}

\newcommand{\paperVersion}{arxiv}  

\maketitle

\begin{abstract}
  Self-attention is usually described as a flexible, content-adaptive way to mix a token with information from its past. We reinterpret causal self-attention transformers, the backbone of modern foundation models, within a probabilistic framework, much as classical PCA is extended to probabilistic PCA. This reformulation reveals a key structural consequence of the underlying change of variables: a barrier constraint emerges on the parameters of self-attention. The resulting geometry exposes a degeneracy boundary where the attention-induced mapping becomes locally ill-conditioned, yielding a \emph{stability-margin} interpretation analogous to the margin in support vector machines. This, in turn, naturally gives rise to the concept of \emph{support tokens}.

We further show that causal transformers define a consistent stochastic process over infinite token sequences, providing a rigorous probabilistic foundation for sequence modeling. Building on this view, we derive a Bayesian MAP training objective that requires only a minimal modification to standard LLM training: adding a smooth log-barrier penalty to the usual cross-entropy loss. Empirically, the resulting training objective improves robustness to input perturbations and sharpens the margin geometry of the learned representations without sacrificing out-of-sample accuracy.

\end{abstract}



\section{Introduction}
\label{sec:introduction}

Transformers are the default architecture for sequence modeling across language, vision, and many other domains \citep{VaswaniShPaUsJoGoKaPo17,BrownMaRyEtAl20,DosovitskiyBeKoWeZhUnDe21}. The heart of a transformer model is causal self-attention, which is often described as a content-adaptive weighted average: each token mixes information from its past according to similarity scores. Moving beyond this colloquial description, we ask a more formal question:

\begin{quote}
\emph{Does causal self-attention admit an explicit probabilistic interpretation, and if so, what does that interpretation imply about the geometry and inductive bias of the model?}
\end{quote}

In this paper, we give a concrete answer. We show that causal self-attention can be interpreted as a probabilistic model on embeddings (treated as latent variables). A key consequence of this probabilistic formulation is a barrier constraint on the parameters of self-attention. 
This constraint induces a geometry with important ramifications: it reveals boundaries in the latent space where the attention-induced mapping becomes locally ill-conditioned. 
Using these boundaries, we introduce the \emph{margin to degeneracy} --- a novel measure of how far each sequence position is from an unstable boundary.

Crucially, our theory reveals that this geometry is dictated by the sign of the effective attention coupling. Under negative coupling, this boundary disappears, and the model promotes sequence dispersion. However, under positive coupling, a hard barrier emerges that explicitly penalizes high attention-weighted variance. Because this positive-coupling regime naturally yields a large-margin interpretation, it forms the primary focus of our framework. 

In this barrier regime, the stability of a sequence is not uniform; it is governed by its most vulnerable positions. Tokens whose contexts drive them closest to the degeneracy boundary act as bottlenecks. Because they dominate the barrier penalty in the log-likelihood, these specific tokens dictate the geometric conditioning of the entire sequence. We term these \emph{support tokens}, drawing a direct analogy to support vectors in large-margin classifiers \citep{CortesVa95}: just as support vectors are the critical points that define a decision boundary, support tokens are the critical sequence positions that define the stability margin (Figure~\ref{fig:margin_schematic}).

\begin{figure}[H]
    \centering
    \includegraphics[width=0.85\linewidth]{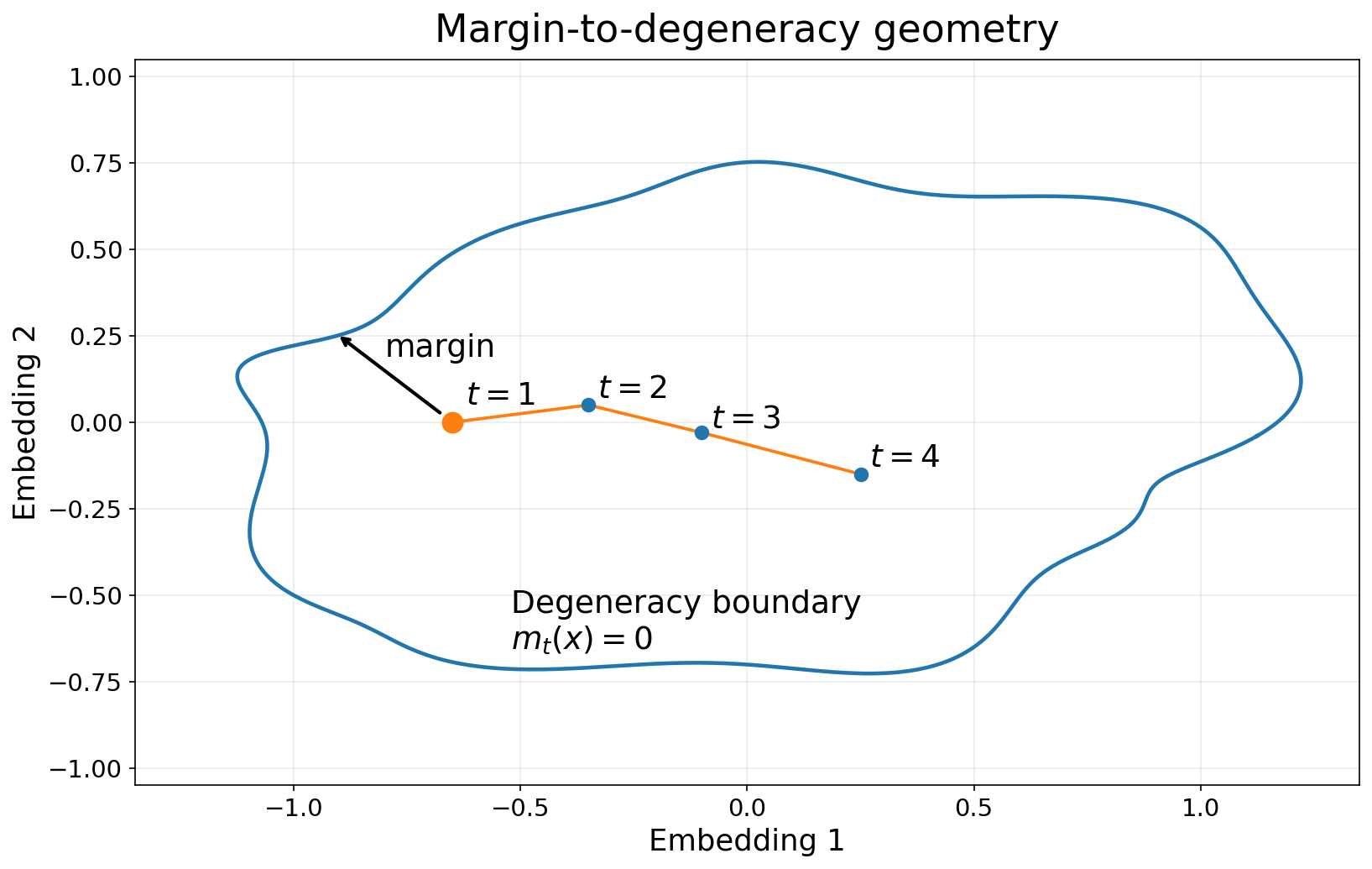}
    \caption{\textbf{Margin-to-degeneracy geometry (SVM analogy).}
The curve denotes the degeneracy boundary: configurations where the attention-induced latent-noise mapping becomes locally ill-conditioned, so small perturbations produce unstable changes.
A sequence traces a trajectory of token representations in embedding space across positions (dots).
Schematically, a token's margin is represented by its distance to this boundary (arrow).
In this example, \textbf{token 1 is the sequence's support token} because it lies closest to the boundary, so it controls the sequence-level stability margin and carries the largest barrier pressure.
This mirrors the SVM picture, where support vectors are the points closest to the decision boundary and therefore determine the margin. A more precise formulation is given in Section~\ref{sec:single_layer_margin}.
}
    \label{fig:margin_schematic}
\end{figure}

Practically, our probabilistic formulation produces an immediate optimization objective. We work in a Bayesian framework \citep{Bishop06,Murphy12} where the usual transformer model of tokens conditioned on embeddings is treated as a log-likelihood. We further compose this model with a log-prior based on a probabilistic causal transformer we propose in this paper. Formulating a log-posterior for this model and performing maximum a posteriori (MAP) estimation yields a familiar data-fit term alongside a novel barrier term that enforces geometric stability. Our approach connects causal self-attention to classical constrained estimation \citep{BoydVa04}: maximizing the posterior is equivalent to optimizing a squared-error objective (at noise scale $\sigma$) under implicit margin constraints that keep the attention mapping well-behaved.

Beyond the single-layer self-attention mechanism, we also clarify how depth composes. In particular, we show that deep transformers form a hierarchical probability model across layers, and that the induced token distributions are consistent across sequence lengths, defining a well-posed stochastic process over token sequences.

\paragraph{Contributions.}

\begin{enumerate}
    \item \textbf{A probabilistic interpretation of causal self-attention.}
    We formalize a causal self-attention layer as a conditional probability model over latent embeddings, which in turn induces a joint probability law over token sequences with an exact likelihood.

    \item \textbf{Margin to degeneracy and a stability-inducing log-barrier.}
    We show that the induced likelihood contains an additional term that defines a margin to a critical degeneracy boundary and yields a smooth barrier against locally unstable attention geometries (Theorem~\ref{thm:jacobian_dispersion}).

    \item \textbf{An optimization view.}
    We show that our probabilistic log-likelihood is analytically equivalent to a standard squared-error objective combined with a margin-based barrier term. Thus, our framework does not replace standard transformers, but rather reinterprets their existing dynamics through a geometric lens.

    \item \textbf{A model-implied training penalty and empirical validation.}
    Using MAP estimation, we translate this theoretical barrier into a practical, plug-and-play training penalty that requires no architectural modifications. We train standard language models with this augmented objective and systematically assess the effect of different margin penalties. Our empirical evaluations demonstrate that explicitly enforcing this geometric margin concretely improves the robustness of the trained models.

    \item \textbf{Depth as a hierarchy of conditional priors.}
    We characterize how the probabilistic interpretation composes across depth and identify conditions under which the nontrivial volume correction localizes to the first layer under standard transformer conditioning. This explains why a one-layer probabilistic causal prior is already a natural practical choice in a Bayesian setting.
\end{enumerate}

\paragraph{Organization.}
Section~\ref{sec:graphical_model} introduces our latent-noise view of causal self-attention as a structural (graphical) model over embeddings.
Section~\ref{sec:single_layer_margin} develops the single-layer theory and derives the margin/log-barrier term induced by token-dependent attention.
Section~\ref{sec:training_objective} bridges this probabilistic perspective to optimization, deriving a squared-error training objective with a stability constraint.
Section~\ref{sec:stochastic_process} establishes that the induced family of token distributions is consistent across sequence lengths (i.e., defines a well-posed stochastic process), which is an important technical ingredient for learning from datasets of variable-length sequences.
Section~\ref{sec:deep_hierarchy} extends the interpretation to depth, showing how layers compose into a hierarchy of conditional priors and how the correction terms behave under standard conditioning.
Section~\ref{sec:experiments} trains models with the resulting objective and evaluates them empirically.
Section~\ref{sec:related_work} surveys related work, and Section~\ref{sec:discussion} discusses implications and future directions.
All proofs are deferred to the Appendix.

\section{A Latent-Noise View of Causal Self-Attention}
\label{sec:graphical_model}

This section introduces the probabilistic formulation for causal self-attention that forms the core of this paper. The goal is not to modify the transformer architecture, but to make explicit the distribution over the model's \emph{embeddings} (hidden states) that a causal self-attention layer implicitly defines when viewed as a generative mechanism.

Fix a sequence length $L$ and an embedding dimension $d$. We write
\[
x_{1:L} = (x_1,\dots,x_L), \qquad x_t \in \mathbb{R}^d,
\]
for the sequence of continuous embeddings produced by a transformer \citep{VaswaniShPaUsJoGoKaPo17,BrownMaRyEtAl20}. When we need to refer to depth, we write $x_t^{(\ell)}$ for the layer-$\ell$ embedding, with $x_t^{(0)}$ denoting the input token embedding. We assume, without loss of generality, that the embeddings are detrended, as the analysis applies equally to residuals. For clarity, we omit the layer index $(\ell)$ in sections focusing on the single-layer setting.

The transformer is ultimately used to define a probability model over discrete tokens via cross-entropy loss, but its computation proceeds through these continuous embeddings. Our focus in this paper is: \emph{what probabilistic structure does causal self-attention induce over the embedding sequence $x_{1:L}$?}

We adopt a probabilistic perspective: we treat embeddings as \emph{random variables} rather than fixed activations. Concretely, we posit that embeddings are generated sequentially from latent noise,
\[
\varepsilon_{1:L} = (\varepsilon_1,\dots,\varepsilon_L), \qquad \varepsilon_t \in \mathbb{R}^d,
\]
through a causal transformation. This is the same conceptual step that turns deterministic estimators into probabilistic models (e.g., probabilistic PCA \citep{TippingBi99}; see also standard latent-variable modeling treatments \citep{Bishop06,Murphy12}).

Throughout the main text, we use an isotropic Gaussian base noise:
\begin{equation}
\varepsilon_t \sim \mathcal{N}(0,\sigma^2 I_d),
\label{eq:base_noise}
\end{equation}
where $\sigma>0$ controls the noise scale. In the limit $\sigma\to 0$, the usual deterministic view is recovered; at finite $\sigma$, the probabilistic formulation reveals additional geometric structure.

\subsection{Causal self-attention as a probabilistic model over embeddings}
Consider a single self-attention layer. For clarity, we start with a single head and write the usual query/key/value projections:
\[
q_t = W_Q x_t, 
\qquad 
k_s = W_K x_s, 
\qquad 
v_s = W_V x_s.
\]
Causal self-attention forms a context summary $\mu_t$ by attending to the past. In the \emph{strict causal} form \citep{VaswaniShPaUsJoGoKaPo17,BrownMaRyEtAl20},
\begin{equation}
\alpha_{ts}(x)
\;=\;
\frac{\exp\!\left(q_t^\top k_s\right)}{\sum_{r<t}\exp\!\left(q_t^\top k_r\right)},
\qquad s<t,
\qquad
\mu_t(x)
\;=\;
\sum_{s<t} \alpha_{ts}(x)\, v_s .
\label{eq:attn_summary_strict}
\end{equation}
The standard deterministic view treats $\mu_t$ as an update computed from the current embedding. Our probabilistic view instead treats $\mu_t$ as the \emph{mean} of a probability model for $x_t$ given its context, with additive latent noise:
\begin{equation}
x_t \;=\; \mu_t(x) \;+\; \varepsilon_t,
\qquad
\varepsilon_t \sim \mathcal{N}(0,\sigma^2 I_d).
\label{eq:latent_noise_model}
\end{equation}
The key modeling point is that $\mu_t(x)$ is both \emph{causal} (it uses only earlier positions) and \emph{token-dependent}: the weights $\alpha_{ts}(x)$ depend on the current token through the query $q_t=W_Qx_t$.

It is convenient to write \eqref{eq:latent_noise_model} in residual form:
\begin{equation}
\varepsilon_t \;=\; e_t(x) \;\triangleq\; x_t - \mu_t(x).
\label{eq:residual_map_graphical}
\end{equation}
This defines a transformation between latent variables $\varepsilon_{1:L}$ and embeddings $x_{1:L}$, and it is this transformation that we analyze throughout the paper.

While standard implementations often include the current position in the context ($s \le t$) \citep{RadfordWuChEtAl19,BrownMaRyEtAl20}, we use the strict-causal setup \eqref{eq:attn_summary_strict} here purely for expository clarity. Our core results and the induced stability term apply equally when $s = t$; the straightforward modifications for this case are detailed in Appendix~\ref{app:self_in_context}.

\subsection{The crucial log-Jacobian}
Because $\varepsilon_{1:L}$ has a tractable base density \eqref{eq:base_noise}, the transformation \eqref{eq:residual_map_graphical} induces an explicit density over embeddings whenever it is locally invertible. Following standard normalizing flow treatments \citep{PapamakariosNaReMoLa21}, the change-of-variables formula gives:
\begin{equation}
\log p(x_{1:L})
\;=\;
\log p(\varepsilon_{1:L})
\;+\;
\log \left|\det J_{x \mapsto \varepsilon}(x_{1:L})\right|.
\label{eq:cov_graphical}
\end{equation}
The first term is the familiar ``noise energy'' (a prediction-error term under the Gaussian choice). The second term is the crucial one: it accounts for local volume change under the attention-induced transformation.

Crucially, this Jacobian term does not vanish because attention is explicitly \emph{token-dependent}. Even though $\mu_t$ summarizes the past, the weights used to summarize it depend on the current token through $q_t=W_Qx_t$. As a result, the residual map $e_t(x)=x_t-\mu_t(x)$ is not simply ``$x_t$ minus a past-only average''; it is a token-dependent reparameterization that generically induces a nontrivial volume factor in the exact likelihood.

In the next section, we compute this volume factor explicitly and show that it defines a \emph{margin to degeneracy}: it quantifies how far the attention-induced map is from a critical surface where it becomes locally singular. This yields a smooth barrier-like term in the log-likelihood that shapes attention geometry in a principled way.

\section{Single-Layer Geometry: Margin and the Context-Dispersion Penalty}
\label{sec:single_layer_margin}

This section makes the probabilistic interpretation concrete for the simplest nontrivial setup: a \emph{single} causal self-attention layer viewed as a latent-noise generator of embeddings.
The key takeaway is that token-dependent attention induces an \emph{additional} term for the log-probability density that diverges to $-\infty$ near a degeneracy boundary (where it becomes locally singular), thus behaving like a smooth \emph{log-barrier} \citep{BoydVa04} that disfavors unstable configurations.
We derive this effect in the scalar case ($d=1$), where the geometry is transparent, and then state the general $d>1$ result, where scalar variance becomes an attention-weighted covariance matrix.
All proofs are deferred to Appendix~\ref{app:proofs_single_layer}.

We work with the latent-noise formulation from Section~\ref{sec:graphical_model}.
A single strict-causal self-attention layer defines a context summary $\mu_t(x)$ from the past, and we study the induced residual map
\begin{equation}
\varepsilon_t \;=\; e_t(x) \;\triangleq\; x_t - \mu_t(x),
\label{eq:residual_map_margin}
\end{equation}
under a tractable base density on $\varepsilon_{1:L}$ (here $\varepsilon_t \sim \mathcal{N}(0,\sigma^2 I_d)$).
The crucial point is that $\mu_t(x)$ depends on $x_t$ through the query, so $x\mapsto \varepsilon$ is a token-dependent transformation; its local sensitivity depends on the configuration of the attended context.

\subsection{Simple case: the scalar  (\texorpdfstring{$d=1$}{d=1})}
\label{subsec:scalar_case}

We begin with $d=1$, so $x_t,v_s\in\mathbb{R}$.
Let $\alpha_{ts}$ denote the strict-causal attention weights, and define the attention-weighted mean and variance of the attended values:
\begin{equation}
\bar v_t \;\triangleq\; \sum_{s<t}\alpha_{ts} v_s,
\qquad
\mathrm{Var}_t \;\triangleq\; \sum_{s<t}\alpha_{ts}(v_s-\bar v_t)^2.
\label{eq:scalar_var_def}
\end{equation}
Because $\alpha_{ts}$ depends on the current token through $q_t=W_Q x_t$, the residual map $e_t(x)=x_t-\mu_t(x)$ has a nontrivial diagonal derivative.

\begin{proposition}[Scalar diagonal derivative]
\label{prop:scalar_diag}
Assume $d=1$ and bilinear logits $q_t^\top k_s$ with $q_t=W_Q x_t$ and $k_s=W_K x_s$. Then for each $t$,
\begin{equation}
\frac{\partial e_t}{\partial x_t}
\;=\;
1 - a\,\mathrm{Var}_t,
\qquad
a \triangleq W_K W_Q .
\label{eq:scalar_diag_deriv}
\end{equation}
\end{proposition}

Equation~\eqref{eq:scalar_diag_deriv} already captures the central phenomenon: the local sensitivity of the map $x_t\mapsto \varepsilon_t$ is governed by the attention-weighted dispersion of the attended context.
Importantly, the coefficient $a$ can have either sign, so the way dispersion affects sensitivity depends on the effective query-key parameterization; what is universal is that the map becomes unstable when the derivative approaches zero.

\subsection{Margin to degeneracy and ``support tokens''}
\label{subsec:margin_support_tokens}

Proposition~\ref{prop:scalar_diag} motivates defining a position-wise \emph{margin to degeneracy}:
\begin{equation}
m_t(x) \;\triangleq\; 1 - a\,\mathrm{Var}_t(x).
\label{eq:margin_def}
\end{equation}
The attention-induced map becomes locally singular at position $t$ precisely when $m_t(x)=0$, i.e., when $\partial e_t/\partial x_t = 0$.
This defines a \emph{degeneracy boundary} (a critical surface) in embedding space.

To summarize stability across the whole sequence, we define the \emph{sequence margin}
\begin{equation}
m(x) \;\triangleq\; \min_{t\in\{1,\dots,L\}} m_t(x).
\label{eq:global_margin}
\end{equation}
Indices attaining (or nearly attaining) this minimum are the bottlenecks for stability.

\begin{definition}[Support tokens]
\label{def:support_tokens}
Any index $t^\star \in \arg\min_t m_t(x)$ is called a \emph{support token}. These are the positions whose attention geometry lies closest to the degeneracy boundary and therefore most strongly constrains the global margin.
\end{definition}

This is directly reminiscent of large-margin classifiers, where a small subset of points (support vectors) controls the margin \citep{CortesVa95}. Here, support tokens are the positions that dominate the stability pressure induced by the probability distribution.

\subsection{Decomposition of log-density and the log-barrier term}
\label{subsec:scalar_likelihood}

Assume we are in a locally invertible regime where $m_t(x)\neq 0$ for all $t$ (and in particular, for the simplest ``stable'' case, $m_t(x)>0$).
Under the Gaussian noise model $\varepsilon_t \sim \mathcal{N}(0,\sigma^2)$, the change-of-variables formula \citep{PapamakariosNaReMoLa21} yields the induced log-density over embeddings:
\begin{equation}
\log p(x_{1:L})
\;=\;
-\frac{1}{2\sigma^2}\sum_{t=1}^L e_t(x)^2
\;+\;
\sum_{t=1}^L \log\left| \frac{\partial e_t}{\partial x_t}\right|
\;-\;
L\log\sigma
\;+\; \mathrm{const}.
\label{eq:scalar_loglik_decomp}
\end{equation}
Using Proposition~\ref{prop:scalar_diag}, the additional term becomes
\begin{equation}
\sum_{t=1}^L \log\left| 1 - a\,\mathrm{Var}_t(x)\right|
\;=\;
\sum_{t=1}^L \log|m_t(x)|.
\label{eq:scalar_logdet_margin}
\end{equation}
This is the \emph{missing} term compared to a pure prediction-error objective.
Its role is a smooth log-barrier: as any margin approaches the degeneracy boundary $m_t(x)\to 0$, we have $\log|m_t(x)|\to -\infty$, sharply down-weighting such configurations.
Thus, the exact log-probability does not merely fit residuals; it also prefers attention geometries that keep the token-dependent map away from local singularities.

\subsection{Theory illustration: positive vs.\ negative coupling in the scalar case}
\label{subsec:theory_illustration_scalar}

The scalar formula \eqref{eq:scalar_diag_deriv} already shows that the sign of the effective coupling
$a=W_KW_Q$ determines the qualitative behavior of the additional likelihood term.
When $a>0$, the margin
\[
m_t(x)=1-a\,\mathrm{Var}_t(x)
\]
decreases with attention-weighted variance and can approach zero, producing a true
log-barrier against degeneracy.
When $a<0$, the same factor becomes $1+|a|\,\mathrm{Var}_t(x)$, which stays strictly positive:
the singular boundary disappears, and the change-of-variables term instead amplifies configurations
with larger dispersion rather than excluding them.

\begin{figure}[H]
\centering
\includegraphics[width=\linewidth]{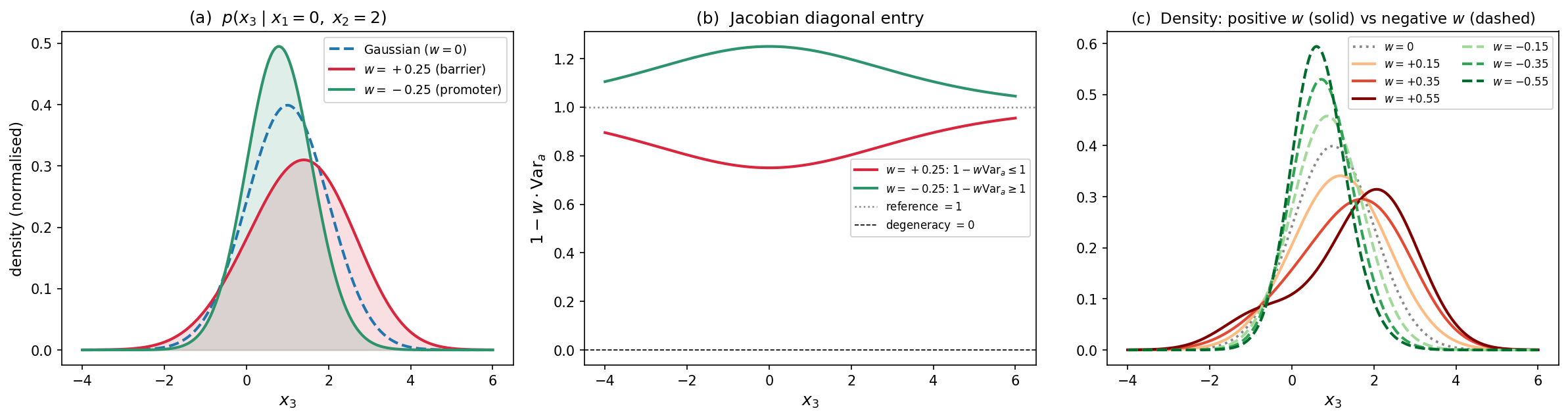}
\caption{\textbf{Prior density under positive and negative coupling ($d{=}1$, $n{=}3$).}
(a)~Conditional density of $x_3$ given $x_1{=}0$, $x_2{=}2$. Positive coupling ($a{=}+0.25$, red) broadens the baseline Gaussian distribution, while negative coupling ($a{=}-0.25$, green) sharpens it.
(b)~The diagonal change-of-variables factor $1-a\,\mathrm{Var}_t$ (Eq.~\ref{eq:scalar_diag_deriv}) evaluated across $x_3$.
(c)~Density profiles for varying coupling strengths $a \in \{-0.55,-0.35,0,+0.35,+0.55\}$.
}
\label{fig:theory_density}
\end{figure}

\begin{figure}[H]
\centering
\includegraphics[width=\linewidth]{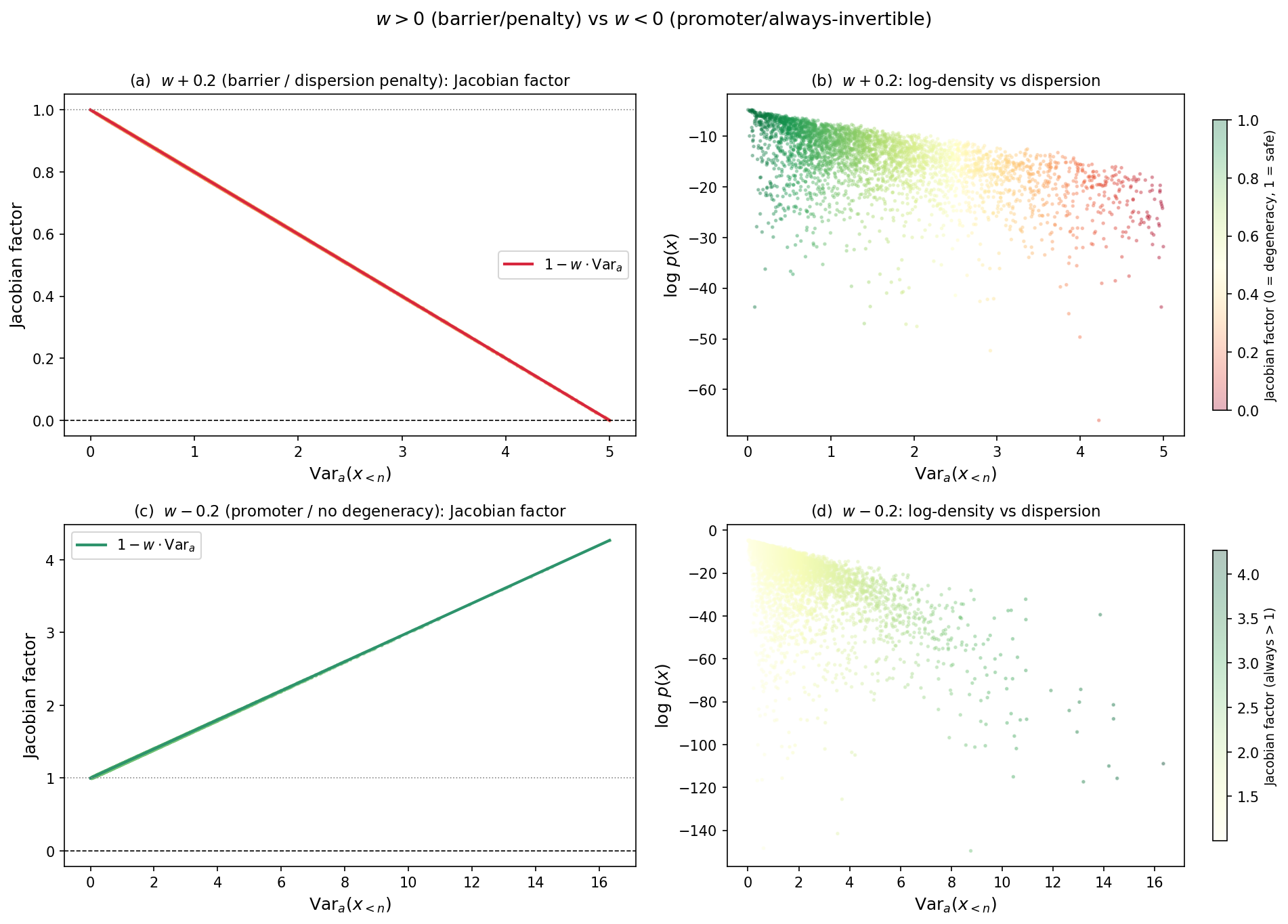}
\caption{\textbf{Population-level behavior under positive vs.\ negative coupling ($d{=}1$, $n{=}5$).}
Evaluation of 4000 random sequences from $\mathcal{N}(0,4)$ under positive coupling ($a{=}+0.2$, top row) and negative coupling ($a{=}-0.2$, bottom row).
(a, c)~The diagonal change-of-variables factor as a function of the attention-weighted variance $\mathrm{Var}_t$. Positive coupling introduces a hard degeneracy boundary at zero, whereas negative coupling remains strictly positive.
(b, d)~The resulting log-density. Under positive coupling, configurations near the degeneracy boundary are excluded (red points). Under negative coupling, the log-barrier is removed and larger dispersion is permitted.
}
\label{fig:theory_locality}
\end{figure}

Figures \ref{fig:theory_density} and \ref{fig:theory_locality} illustrate these two regimes in the simplest scalar setting. We fix $x_1=0$, $x_2=2$, sweep $x_3\in[-4,6]$, and compare the induced conditional density under positive, zero, and negative coupling. We also evaluate 4000 random length-5 sequences to visualize the population-level relationship between $\mathrm{Var}_t$, the diagonal factor, and the log-density.

Figure~\ref{fig:theory_density} shows the local effect on the conditional density.
For positive coupling ($a=+0.25$), the change-of-variables factor is less than one whenever
$\mathrm{Var}_t>0$, so it subtracts from the log-density and flattens the peak:
the maximum density drops from approximately $0.38$ (Gaussian reference) to $0.32$.
For negative coupling ($a=-0.25$), the factor is greater than one, so it adds to the log-density
and sharpens the peak to approximately $0.52$.
As expected from \eqref{eq:scalar_diag_deriv}, the diagonal factor follows the exact linear form
$1-a\,\mathrm{Var}_t$, and increasing $|a|$ makes the deviation from the Gaussian more pronounced.

Figure~\ref{fig:theory_locality} shows the corresponding population-level behavior.
For positive coupling ($a=+0.2$), the factor decreases linearly toward zero and
613 out of 4000 sequences (15.3\%) are excluded by the degeneracy condition; among the remaining
3387 valid sequences, log-density decreases as $\mathrm{Var}_t$ increases.
For negative coupling ($a=-0.2$), all 4000 sequences remain valid, the diagonal factor increases as
$1+|a|\,\mathrm{Var}_t$, and the observed range of $\mathrm{Var}_t$ extends to approximately $16$
(compared with approximately $5$ in the positive-coupling case).
The log-density still decreases overall because the Gaussian residual term dominates in the tails,
but the change-of-variables term now partially offsets that decay instead of imposing a hard barrier.

These illustrations make the sign dependence explicit:
\emph{positive coupling yields a genuine margin-to-degeneracy and a log-barrier, while negative coupling
removes the degeneracy boundary and turns the same term into a dispersion-promoting correction.}

\subsection{General case \texorpdfstring{$d>1$}{d>1}: covariance and a matrix margin}
\label{subsec:matrix_case}

We now state the corresponding result in $d>1$.
Define the attention-weighted mean and covariance of the attended values:
\begin{equation}
\bar v_t \;\triangleq\; \sum_{s<t}\alpha_{ts}(x)\,v_s,
\qquad
\Sigma_t \;\triangleq\; \sum_{s<t}\alpha_{ts}(x)\,(v_s-\bar v_t)(v_s-\bar v_t)^\top \succeq 0.
\label{eq:cov_def_general}
\end{equation}

\begin{theorem}[Diagonal Jacobian block and conditioning]
\label{thm:jacobian_dispersion}
For the residual map $e_t(x)=x_t-\mu_t(x)$ induced by strict-causal self-attention (Section~\ref{sec:graphical_model}), the diagonal Jacobian block satisfies
\begin{equation}
\frac{\partial e_t}{\partial x_t}
\;=\;
I_d \;-\; \Sigma_t A,
\label{eq:matrix_diag_block}
\end{equation}
where $A$ is an explicit matrix determined by the query/key parameterization. In the basic bilinear logit model with $q_t=W_Q x_t$ and $k_s=W_K x_s$, one has
\begin{equation}
A \;=\; W_K^\top W_Q.
\label{eq:A_basic}
\end{equation}
\end{theorem}

This identifies a natural matrix-valued notion of ``distance to degeneracy'': the map becomes locally singular at $t$ precisely when $\det(I_d-\Sigma_tA)=0$.
Equivalently, stability can be expressed via spectral conditions on $\Sigma_tA$ (formalized in Appendix~\ref{app:spectral_condition}).

\begin{corollary}[Log-likelihood and the context-geometry term]
\label{cor:logdet_penalty}
Assume we are in a locally invertible regime where $\det(I_d-\Sigma_tA)\neq 0$ for all $t$.
Under $\varepsilon_t \sim \mathcal{N}(0,\sigma^2 I_d)$, the induced log-density over embeddings decomposes as
\begin{equation}
\log p(x_{1:L})
\;=\;
-\frac{1}{2\sigma^2}\sum_{t=1}^L \|e_t(x)\|_2^2
\;+\;
\sum_{t=1}^L \log\left|\det\!\left(I_d-\Sigma_t A\right)\right|
\;-\;
Ld\log\sigma
\;+\; \mathrm{const}.
\label{eq:loglik_matrix_decomp}
\end{equation}
The second term is the additional geometry/conditioning term induced by token-dependent attention: it diverges to $-\infty$ as $\det(I_d-\Sigma_tA)\to 0$, acting as a smooth barrier against locally singular attention-induced transformations.
\end{corollary}


\section{An Optimization View: Squared Error Under a Stability Margin}
\label{sec:training_objective}

The previous sections define an explicit probabilistic model over embedding sequences:
latent noise $\varepsilon_{1:L}$ is mapped to embeddings $x_{1:L}$ through the
token-dependent causal attention transformation, yielding an exact log-probability via the change-of-variables formula
\citep{PapamakariosNaReMoLa21}.
This section takes an \emph{optimization lens} on that density.
The key result is simple: \emph{maximizing the embedding prior likelihood is equivalent to a
squared-error objective (at noise scale $\sigma$) together with a stability term that acts as a
log-barrier against approaching the degeneracy boundary} \citep{BoydVa04}.
We present the scalar case first, where the constraint has an explicit ``margin'' form,
and then state the general $d$-dimensional counterpart.

Recall the latent-noise generative rule
\begin{equation}
x_t = \mu_t(x) + \varepsilon_t,
\qquad
\varepsilon_t \sim \mathcal{N}(0,\sigma^2 I_d),
\label{eq:latent_noise_sigma_recall}
\end{equation}
with residual map $e_t(x)=x_t-\mu_t(x)$ and the induced embedding density
(obtained by change-of-variables \citep{PapamakariosNaReMoLa21}):
\begin{equation}
\log p_\sigma(x_{1:L})
=
-\frac{1}{2\sigma^2}\sum_{t=1}^L \|e_t(x)\|_2^2
+
\sum_{t=1}^L \log\left|\det\!\left(\frac{\partial e_t}{\partial x_t}\right)\right|
-\;Ld\log\sigma
+\mathrm{const}.
\label{eq:logpx_general_sigma}
\end{equation}
Here the first term is a familiar squared-error fit to the attention prediction $\mu_t(x)$,
and the second term is the stability/conditioning correction induced by token-dependent attention
(Section~\ref{sec:single_layer_margin}).

Maximizing $\log p_\sigma(x_{1:L})$ is equivalent (up to additive constants and a scaling)
to minimizing the negative log-prior:
\begin{equation}
\mathcal{L}_{\mathrm{prior},\sigma}(x)
\;\triangleq\;
\frac{1}{2\sigma^2}\sum_{t=1}^L \|x_t-\mu_t(x)\|_2^2
\;-\;
\sum_{t=1}^L \log\left|\det\!\left(\frac{\partial e_t}{\partial x_t}\right)\right|.
\label{eq:prior_objective_def}
\end{equation}
We interpret \eqref{eq:prior_objective_def} as \emph{squared error under a stability margin}:
the first term encourages embeddings to be close to their attention-predicted contexts,
while the second discourages configurations where the attention-induced mapping becomes ill-conditioned.

\subsection{The scalar objective (\texorpdfstring{$d=1$}{d=1})}

In the scalar case ($d=1$), Section~\ref{sec:single_layer_margin} showed that
\begin{equation}
\frac{\partial e_t}{\partial x_t}
=
m_t(x)
\;\triangleq\;
1-a\,\mathrm{Var}_t(x),
\label{eq:mt_recall}
\end{equation}
where $\mathrm{Var}_t$ is the attention-weighted variance of the attended values and
$a=W_KW_Q$ is the scalar analogue of $W_K^\top W_Q$.

Substituting \eqref{eq:mt_recall} into \eqref{eq:prior_objective_def} yields a particularly
transparent objective:
\begin{equation}
\mathcal{L}_{\mathrm{prior},\sigma}(x)
=
\frac{1}{2\sigma^2}\sum_{t=1}^L \big(x_t-\mu_t(x)\big)^2
\;-\;
\sum_{t=1}^L \log|m_t(x)|
\quad (+\;\text{const}).
\label{eq:scalar_prior_objective}
\end{equation}

\paragraph{Barrier interpretation.}
The set $\{x : m_t(x)=0\}$ is the \emph{degeneracy boundary} at position $t$:
the residual map loses local invertibility there.
Accordingly, $-\log|m_t(x)|$ acts as a smooth barrier that grows without bound as $m_t(x)\to 0$
\citep{BoydVa04}.
In optimization terms, likelihood maximization discourages sequences whose attention geometry pushes
the map toward singularity.

\paragraph{An explicit constrained form (``margin constraint'').}
Equation \eqref{eq:scalar_prior_objective} can be viewed as a softened constrained problem:
for any target margin level $\delta>0$, consider
\begin{equation}
\min_x \;\;\frac{1}{2\sigma^2}\sum_{t=1}^L \big(x_t-\mu_t(x)\big)^2
\qquad\text{s.t.}\qquad
m_t(x)\ge \delta\;\;\;\forall t.
\label{eq:scalar_constrained_problem}
\end{equation}
The log-barrier objective \eqref{eq:scalar_prior_objective} is the classical interior-point relaxation
of \eqref{eq:scalar_constrained_problem} \citep{BoydVa04}:
it replaces hard constraints $m_t(x)\ge \delta$ with a barrier that becomes large as $m_t(x)$ approaches $0$.
This makes the ``stability margin'' viewpoint precise: likelihood maximization trades off data fit
against maintaining a positive margin away from degeneracy.

\subsection{The general matrix objective (\texorpdfstring{$d>1$}{d>1})}

For $d>1$, Theorem~\ref{thm:jacobian_dispersion} gives the diagonal block
\begin{equation}
\frac{\partial e_t}{\partial x_t} = I_d-\Sigma_t A,
\qquad
A=W_K^\top W_Q,
\label{eq:block_recall}
\end{equation}
with $\Sigma_t$ the attention-weighted covariance of the attended values.
Substituting into \eqref{eq:prior_objective_def} yields
\begin{equation}
\mathcal{L}_{\mathrm{prior},\sigma}(x)
=
\frac{1}{2\sigma^2}\sum_{t=1}^L \|x_t-\mu_t(x)\|_2^2
\;-\;
\sum_{t=1}^L \log\left|\det\!\left(I_d-\Sigma_t A\right)\right|
\quad (+\;\text{const}).
\label{eq:matrix_prior_objective}
\end{equation}

\paragraph{Matrix margin.}
The degeneracy boundary is characterized by $\det(I_d-\Sigma_tA)=0$,
i.e., by configurations where $1$ becomes an eigenvalue of $\Sigma_tA$.
This suggests a natural ``margin to degeneracy'' at each position, such as
\begin{equation}
m_t(x)\;\triangleq\; \left|\det(I_d-\Sigma_tA)\right|
\qquad\text{or}\qquad
m_t(x)\;\triangleq\; 1-\rho(\Sigma_tA),
\label{eq:matrix_margin_def_options}
\end{equation}
where $\rho(\cdot)$ denotes the spectral radius.
Different choices are useful in different analyses; the log-likelihood itself selects the
log-determinant form in \eqref{eq:matrix_prior_objective}.

\subsection{Integration into the full token model}
\label{subsec:bridge_to_tokens}

So far we have focused on the embedding prior $p_\sigma(x_{1:L})$ induced by causal attention.
To model actual token sequences $y_{1:L}$, we will combine this prior with the standard transformer
token likelihood \citep{VaswaniShPaUsJoGoKaPo17,BrownMaRyEtAl20}:
\begin{equation}
p(y_{1:L},x_{1:L})
=
p(y_{1:L}\mid x_{1:L})\;p_\sigma(x_{1:L}).
\label{eq:joint_y_x_preview}
\end{equation}
Here $p(y\mid x)$ is exactly what is used in practice: given embeddings (hidden states),
the decoder head produces categorical distributions over tokens at each position.
The prior $p_\sigma(x)$ is the new object made explicit by our latent-noise interpretation. 

\paragraph{MAP training objective.}
In the full token model, we view the categorical decoder as a likelihood and the embedding distribution as a prior.
Given a token sequence $y_{1:L}$, we estimate parameters by \emph{maximum a posteriori} (MAP) estimation:
\begin{equation}
(\hat\theta,\hat\phi)
\;\in\;
\arg\max_{\theta,\phi}\;
\log P_\theta(y_{1:L}\mid x_{1:L})
\;+\;
\log p_{\phi,\sigma}(x_{1:L}),
\label{eq:map_training}
\end{equation}
which is equivalent to minimizing the negative log-posterior---precisely the standard cross-entropy loss augmented with our prior-induced stability term. The key practical insight is that LLMs can be fit using standard training machinery; the only modification is the addition of this log-prior regularizer, which is straightforward to incorporate and actively promotes robust model fitting.

In the next section (Section~\ref{sec:stochastic_process}), we formalize the joint model $p(y\mid x)\,p_\sigma(x)$. We show that marginalizing out the embeddings yields a well-defined sequence likelihood $p(y)$. Crucially, we establish that this distribution satisfies Kolmogorov consistency across sequence lengths \citep{Billingsley95}. This property ensures the model assigns coherent probabilities to variable-length datasets, firmly grounding our training objective in stochastic process theory.


\section{A Token-Level Stochastic Process Induced by the Embedding Prior}
\label{sec:stochastic_process}

So far, we have constructed an explicit embedding prior $p_\sigma(x_{1:n})$ induced by causal self-attention. To connect this to real transformer language modeling, we now combine this prior with the usual categorical decoder likelihood $P(y\mid x)$ \citep{VaswaniShPaUsJoGoKaPo17,BrownMaRyEtAl20} and study the induced token law $\{P(y_{1:n})\}_{n\ge 1}$.

Our main result is that, under strict causality, these finite-dimensional token distributions are consistent across sequence lengths and therefore define a bona fide stochastic process over infinite token sequences via standard extension theorems \citep{Billingsley95,Kolmogorov33}. This gives a rigorous probabilistic foundation for viewing a causal transformer as a single law over variable-length token sequences. As a note, all measure-theoretic proofs are deferred to Appendix~\ref{app:kolmogorov_proofs}.

\subsection{From the embedding prior to a token-level law}
\label{subsec:token_model_from_prior}

Let $V=\{v_1,\dots,v_K\}$ be a finite vocabulary, equipped with the discrete $\sigma$-algebra $2^V$. For each variable sequence length $n\ge 1$, the length-$n$ token space is $V^n$, and the infinite token space is the product $V^\infty=\prod_{t=1}^\infty V$, endowed with the product $\sigma$-algebra generated by cylinder sets \citep{Billingsley95}.

Given latent embeddings $x_{1:n}$, we use the standard transformer decoder likelihood:
\begin{equation}
P(y_{1:n}\mid x_{1:n})
\;=\;
\prod_{t=1}^n P(y_t \mid y_{<t}, x_{1:n}),
\label{eq:obs_factorization}
\end{equation}
where each factor is a categorical distribution over $V$ (typically a softmax computed from the final hidden state at position $t$). This is exactly the observation model used in practice.

We pair this with the causal attention embedding prior developed earlier:
\begin{equation}
x_t \;=\; \mu_t(x) + \varepsilon_t,
\qquad
\varepsilon_t \sim \mathcal N(0,\sigma^2 I_d),
\label{eq:state_equation_recall}
\end{equation}
with strict causal masking in $\mu_t$ (only $s<t$ contribute), and with the Jacobian non-degeneracy assumptions from Section~\ref{sec:single_layer_margin} so that $p_\sigma(x_{1:n})$ is well-defined.

Together, these define the joint model:
\begin{equation}
P(y_{1:n}, x_{1:n})
\;=\;
P(y_{1:n}\mid x_{1:n})\, p_\sigma(x_{1:n}),
\label{eq:joint_yx}
\end{equation}
and hence the marginal token likelihood:
\begin{equation}
P(y_{1:n})
\;=\;
\int_{\mathbb R^{dn}} P(y_{1:n}\mid x_{1:n})\, p_\sigma(x_{1:n}) \, dx_{1:n}.
\label{eq:marginal_tokens}
\end{equation}

Equation~\eqref{eq:marginal_tokens} is the key object in this section: it is the token-level law induced by the probabilistic transformer, obtained by integrating out the latent embeddings.

\subsection{Consistency across sequence lengths}
\label{subsec:consistency_across_lengths}

To learn from datasets containing sequences of different lengths, and to interpret the model as a single probabilistic object rather than a separate model for each $n$, the family $\{P(y_{1:n})\}_{n\ge 1}$ must be \emph{projectively consistent} \citep{Billingsley95}. In the present setting, strict causality is exactly what makes this possible.

\begin{theorem}[Kolmogorov consistency]
\label{thm:kolmogorov_consistency}
Assume strict causal masking and Jacobian non-degeneracy so that $p_\sigma(x_{1:n})$ exists for each $n$. Then for all $n\ge 2$ and all $y_{1:n-1}\in V^{n-1}$,
\begin{equation}
\sum_{y_n\in V} P(y_{1:n}) \;=\; P(y_{1:n-1}),
\label{eq:kolmogorov_consistency}
\end{equation}
where the right-hand side is the marginal likelihood obtained from the \emph{$(n-1)$-dimensional} model.
\end{theorem}

The role of causality is essential: without it, consistency can fail.

\begin{proposition}[Non-causal attention breaks consistency]
\label{prop:noncausal_breaks_consistency}
If attention is non-causal (future positions may influence the attention normalization at earlier positions), then in general the family $\{P(y_{1:n})\}_{n\ge 1}$ induced by \eqref{eq:marginal_tokens} is not projectively consistent, and \eqref{eq:kolmogorov_consistency} need not hold.
\end{proposition}

Taken together, Theorem~\ref{thm:kolmogorov_consistency} and Proposition~\ref{prop:noncausal_breaks_consistency} show that strict causality is not merely a computational convenience here; it is the structural property that makes the induced family of token laws mathematically coherent across sequence lengths.

\subsection{Infinite-sequence law and dataset-level likelihood}
\label{subsec:infinite_law_and_datasets}

Once projective consistency holds, standard extension theory yields a unique probability measure on the infinite product space.

\begin{theorem}[Transformer stochastic process]
\label{thm:kolmogorov_extension}
Under the assumptions of Theorem~\ref{thm:kolmogorov_consistency}, there exists a unique probability measure $P$ on $(V^\infty,\mathcal F)$ such that for every $n\ge 1$ and every $y_{1:n}\in V^n$,
\begin{equation}
P(Y_1=y_1,\dots,Y_n=y_n) \;=\; P(y_{1:n}),
\end{equation}
where $(Y_t)_{t\ge 1}$ are the coordinate projections on $V^\infty$.
\end{theorem}

This immediately gives a principled interpretation of learning on datasets. If $\{y^{(i)}\}_{i=1}^N$ is a collection of variable-length sequences, these can be viewed as finite-length realizations from the same underlying process. One may therefore validly optimize the marginal objective:
\[
\sum_{i=1}^N \log P(y^{(i)}_{1:n_i}),
\]
or, if one wishes to retain latent embeddings explicitly, optimize the augmented objective:
\begin{equation}
\sum_{i=1}^N \log P(y^{(i)}_{1:n_i}, x^{(i)}_{1:n_i})
\;=\;
\sum_{i=1}^N \log P(y^{(i)}_{1:n_i}\mid x^{(i)}_{1:n_i})
\;+\;
\sum_{i=1}^N \log p_\sigma(x^{(i)}_{1:n_i}),
\label{eq:augmented_dataset_objective}
\end{equation}
where the embedding prior term contributes the model-implied stability and margin regularization. This completes the probabilistic basis for the MAP-style estimation in the transformer setting developed in the previous sections.


\section{Depth: \texorpdfstring{$L$}{L}-Layer Composition as Hierarchical Conditional Priors}
\label{sec:deep_hierarchy}

The previous sections analyzed a \emph{single} causal attention map as a latent-noise generator of
embeddings and showed that token-dependent attention induces an additional stability term in the
exact likelihood via the change-of-variables formula \citep{PapamakariosNaReMoLa21}.
We now explain how this probabilistic view composes across depth. This section may be skipped on a first reading without loss of continuity.

The key message is structural:
a deep causal transformer naturally defines a \emph{hierarchy} of conditional priors over layerwise
embeddings. Moreover, under the standard architectural convention --- that attention weights in layer
$\ell$ are computed from the previous-layer embeddings $x^{(\ell-1)}$ as in standard transformer blocks
\citep{VaswaniShPaUsJoGoKaPo17,XiongZoLiSo16} --- the nontrivial stability
correction induced by token-dependent mixing \emph{localizes} to the first such token-dependent map.
This makes it possible to capture the main geometric effect with an embedding-level ``attention prior''
module, while leaving the rest of a deep network unchanged.

\subsection{Setup: layer-wise embeddings and latent noise}
\label{subsec:deep_setup}

Let $L_{\mathrm{layers}}$ denote the number of transformer layers. Fixing a context length $T$ for the forward pass, let
\[
x^{(0)}_{1:T}\in(\mathbb{R}^d)^T
\]
be the input token embeddings (including positional encodings).
For each layer $\ell\in\{1,\dots,L_{\mathrm{layers}}\}$, we introduce latent noise
\[
\varepsilon^{(\ell)}_{1:T}=(\varepsilon^{(\ell)}_1,\dots,\varepsilon^{(\ell)}_T),\qquad
\varepsilon^{(\ell)}_t\in\mathbb{R}^d,\qquad
\varepsilon^{(\ell)}_t\sim\mathcal{N}(0,\sigma_\ell^2 I_d),
\]
and we view the layer output embeddings $x^{(\ell)}_{1:T}$ as generated by a causal layer map plus noise.
To keep the exposition focused on attention, we write a generic layer transformation as
\begin{equation}
x^{(\ell)}_t \;=\; g^{(\ell)}_t\!\big(x^{(\ell-1)}_{1:t}\big) \;+\; \varepsilon^{(\ell)}_t,
\qquad t=1,\dots,T,
\label{eq:deep_layer_gen}
\end{equation}
where causality means $g^{(\ell)}_t$ depends only on the prefix $x^{(\ell-1)}_{1:t}$ (via masking).
In a standard transformer block, $g^{(\ell)}$ comprises (multi-head) causal self-attention applied to
$x^{(\ell-1)}$, followed by an MLP and residual/normalization structure \citep{VaswaniShPaUsJoGoKaPo17,XiongZoLiSo16}; our results below only use the
fact that $g^{(\ell)}$ is \emph{computed from} $x^{(\ell-1)}$ and does not depend on the current-layer
noise $\varepsilon^{(\ell)}$.

Equation \eqref{eq:deep_layer_gen} induces a residual map for each layer,
\begin{equation}
\varepsilon^{(\ell)}_t
\;=\;
e^{(\ell)}_t\!\big(x^{(\ell)},x^{(\ell-1)}\big)
\;\triangleq\;
x^{(\ell)}_t - g^{(\ell)}_t\!\big(x^{(\ell-1)}_{1:t}\big).
\label{eq:deep_residual_map}
\end{equation}
This map is \emph{explicit} (forward-generative) as written: once $x^{(\ell-1)}$ is fixed,
$x^{(\ell)}$ is an additive-noise perturbation around $g^{(\ell)}(x^{(\ell-1)})$.

\subsection{Hierarchical conditional-prior factorization}
\label{subsec:deep_factorization}

Because \eqref{eq:deep_layer_gen} is causal in $t$, it defines a layer-wise conditional prior that factors
over positions:
\begin{equation}
p_{\sigma_\ell}\!\left(x^{(\ell)}_{1:T}\mid x^{(\ell-1)}_{1:T}\right)
\;=\;
\prod_{t=1}^T
\mathcal{N}\!\left(
x^{(\ell)}_t \,;\,
g^{(\ell)}_t(x^{(\ell-1)}_{1:t}),
\;\sigma_\ell^2 I_d
\right).
\label{eq:layer_conditional_prior}
\end{equation}
Stacking layers yields a \emph{hierarchical} latent-variable model over embeddings:
\begin{equation}
p\!\left(x^{(1)}_{1:T},\dots,x^{(L_{\mathrm{layers}})}_{1:T}\mid x^{(0)}_{1:T}\right)
\;=\;
\prod_{\ell=1}^{L_{\mathrm{layers}}}
p_{\sigma_\ell}\!\left(x^{(\ell)}_{1:T}\mid x^{(\ell-1)}_{1:T}\right).
\label{eq:deep_hierarchical_prior}
\end{equation}
Equation \eqref{eq:deep_hierarchical_prior} is the formal sense in which depth composes as a hierarchy
of conditional priors: each layer defines a stochastic transition on embeddings, driven by latent noise,
and the deep model is their composition.

\subsection{Exact likelihood and triangular structure across depth}
\label{subsec:deep_loglik}

As in the single-layer analysis, one can compute the exact embedding density induced by the latent noises
via change-of-variables \citep{PapamakariosNaReMoLa21}. Here it is convenient to work with the full collection of latent variables
\[
\varepsilon \;\triangleq\; \{\varepsilon^{(\ell)}_{1:T}\}_{\ell=1}^{L_{\mathrm{layers}}}
\qquad\text{and}\qquad
x \;\triangleq\; \{x^{(\ell)}_{1:T}\}_{\ell=1}^{L_{\mathrm{layers}}}.
\]
Define the global residual map $E$ by stacking \eqref{eq:deep_residual_map} over all $(\ell,t)$:
\[
\varepsilon = E(x; x^{(0)}).
\]
With causal masking, the Jacobian of $E$ with respect to $x$ is block lower-triangular under a natural
ordering of variables (e.g., increasing $(\ell,t)$). Therefore,
\begin{equation}
\log p(x\mid x^{(0)})
\;=\;
\log p(\varepsilon)
\;+\;
\log\left|\det J_{x\mapsto\varepsilon}(x)\right|
\;=\;
\sum_{\ell,t}\log p(\varepsilon^{(\ell)}_t)
\;+\;
\sum_{\ell,t}\log\left|\det\!\left(\frac{\partial e^{(\ell)}_t}{\partial x^{(\ell)}_t}\right)\right|.
\label{eq:deep_logdet_decomp}
\end{equation}
The proofs of triangularity and the determinant decomposition are standard and are deferred to
Appendix~\ref{app:proofs_deep}.

\subsection{Localization of the stability correction under standard conditioning}
\label{subsec:localization}

The single-layer results showed that the stability/log-determinant term becomes nontrivial when the
context summary depends on the \emph{current} embedding being generated (token-dependent mixing in the
same variable).
In the multi-layer setting \eqref{eq:deep_layer_gen}, the standard transformer convention is that the
attention weights in layer $\ell$ are computed from $x^{(\ell-1)}$, not from $x^{(\ell)}$.
Under this convention, the layer-wise residual map \eqref{eq:deep_residual_map} is affine in $x^{(\ell)}_t$,
which implies the diagonal derivative is the identity.

\begin{proposition}[No layer-wise stability correction under previous-layer conditioning]
\label{prop:no_logdet_deep}
Assume that for each layer $\ell\ge 2$, the layer map $g^{(\ell)}_t$ depends only on
$x^{(\ell-1)}_{1:t}$ and does not depend on $x^{(\ell)}$.
Then for every $t$,
\begin{equation}
\frac{\partial e^{(\ell)}_t}{\partial x^{(\ell)}_t}
\;=\;
I_d,
\qquad\text{hence}\qquad
\log\left|\det\!\left(\frac{\partial e^{(\ell)}_t}{\partial x^{(\ell)}_t}\right)\right|
\;=\;
0.
\label{eq:deep_identity_diag}
\end{equation}
\end{proposition}

Proposition~\ref{prop:no_logdet_deep} says: once the attention pattern used to generate a layer is
computed from the previous layer, that layer contributes \emph{only} the Gaussian energy term
$\|x^{(\ell)}_t-g^{(\ell)}_t(x^{(\ell-1)}_{1:t})\|_2^2$ and no additional stability factor.

Where, then, does the nontrivial stability term live in a deep model?
It lives precisely in whichever stage uses token-dependent mixing \emph{with respect to the variable
being generated}.
In this paper, that stage is the embedding-level attention prior analyzed earlier: a causal attention map
that mixes past \emph{values} using weights that depend on the current \emph{query} formed from the
same embedding variable, producing the diagonal block $I_d-\Sigma_t A$ and the margin/barrier term.

\begin{corollary}[Localization to a single attention-prior stage]
\label{cor:localize_first}
Consider a deep model in which:
(i) the first stage defines an embedding-level attention prior of the form analyzed in
Section~\ref{sec:single_layer_margin}, and
(ii) all subsequent layers satisfy the previous-layer conditioning assumption of
Proposition~\ref{prop:no_logdet_deep}.
Then the total stability/log-determinant contribution in \eqref{eq:deep_logdet_decomp} is exactly the
single-layer stability term from the attention-prior stage; deeper layers contribute zero.
\end{corollary}

Corollary~\ref{cor:localize_first} formalizes the practical takeaway used later in our experiments:
the margin/barrier term implied by token-dependent causal mixing can be captured by a single lightweight
module acting on embeddings, without altering the rest of a deep transformer stack.

\subsection{Multi-head attention, residual pathways, and normalization}
\label{subsec:deep_mha_residual_ln}

The statements in this section were written for a generic causal layer map
$g^{(\ell)}_t(\cdot)$ to keep notation light. We now briefly connect this abstraction to a standard
transformer block with multi-head attention, MLPs, residual connections, and normalization.

A typical pre-norm transformer block (one layer) can be written schematically as
\begin{align}
u^{(\ell)}_{1:T} &= x^{(\ell-1)}_{1:T} + \mathrm{MHA}^{(\ell)}\!\left(\mathrm{LN}\!\left(x^{(\ell-1)}_{1:T}\right)\right), \label{eq:block_mha_residual}\\
\tilde{x}^{(\ell)}_{1:T} &= u^{(\ell)}_{1:T} + \mathrm{MLP}^{(\ell)}\!\left(\mathrm{LN}\!\left(u^{(\ell)}_{1:T}\right)\right), \label{eq:block_mlp_residual}
\end{align}
where $\mathrm{MHA}^{(\ell)}$ denotes masked (causal) multi-head self-attention and $\mathrm{LN}$ is
layer normalization. In our latent-noise view, we can introduce stochasticity at the layer output via
\begin{equation}
x^{(\ell)}_t \;=\; \tilde{x}^{(\ell)}_t \;+\; \varepsilon^{(\ell)}_t,
\qquad
\varepsilon^{(\ell)}_t\sim\mathcal{N}(0,\sigma_\ell^2 I_d).
\label{eq:block_add_noise}
\end{equation}
Equivalently, \eqref{eq:block_mha_residual} to \eqref{eq:block_add_noise} defines the layer map
$g^{(\ell)}_t(x^{(\ell-1)}_{1:t})\triangleq \tilde{x}^{(\ell)}_t$ appearing in
\eqref{eq:deep_layer_gen}.

It is worth noting that Proposition~\ref{prop:no_logdet_deep} (no additional stability/log-determinant correction from deeper layers)
continues to hold for these standard transformer blocks provided the following \emph{conditioning convention} is respected: \emph{For each layer $\ell$, the attention weights used inside $\mathrm{MHA}^{(\ell)}$ are computed from
the previous-layer embeddings (e.g., from $\mathrm{LN}(x^{(\ell-1)})$), and the layer noise
$\varepsilon^{(\ell)}$ enters additively at the output as in \eqref{eq:block_add_noise}.}

Under this convention, the residual map at layer $\ell$ has the form
$e^{(\ell)}_t = x^{(\ell)}_t - g^{(\ell)}_t(x^{(\ell-1)}_{1:t})$ and is \emph{affine} in $x^{(\ell)}_t$,
so $\partial e^{(\ell)}_t/\partial x^{(\ell)}_t = I_d$ and the layer contributes no extra stability factor
in \eqref{eq:deep_logdet_decomp}. The geometric ``margin/log-barrier'' effect therefore appears only in
stages where the attention pattern depends on the \emph{current} embedding variable being generated (as in
our embedding-level attention prior), not in layers whose attention patterns are computed from earlier
representations.

The ``no correction'' result is not a claim that depth is irrelevant; rather, it is a statement about where the exact change-of-variables stability factor appears. If a layer is modified so that its attention weights depend on the current-layer variable $x^{(\ell)}$ being generated (e.g., via an implicit/self-consistent layer definition, or by placing $x^{(\ell)}$ inside the queries used in the softmax for the same layer), then the diagonal derivative $\partial e^{(\ell)}_t/\partial x^{(\ell)}_t$ is no longer the identity. Such a modified layer would contribute its own stability/log-determinant term with an analogous margin-to-degeneracy interpretation. We do not pursue these variants here, but the analysis in Section~\ref{sec:single_layer_margin} can be applied layer-wise.

\section{Experiments}
\label{sec:experiments}
In this section, we empirically validate our theoretical results using controlled datasets. While large-scale model training in industrial settings is beyond the scope of this paper, we hope these initial results encourage further investigation and lead to practical insights for robust foundation models.

\subsection{Shared experimental setup}
\label{sec:exp_setup}

We train on the \texttt{wikitext-103} and \texttt{simplebooks-92} corpora \citep{merity2018analysis, nguyen2019simplebooks} and report results on their validation splits.
\texttt{wikitext-103} contains approximately $541\mathrm{M}$ bytes in its training split and $1.15\mathrm{M}$ bytes in its validation split; \texttt{simplebooks-92} contains approximately $421\mathrm{M}$ bytes in its training split and $0.91\mathrm{M}$ bytes in its validation split. 
For each dataset, we train a byte-pair encoding (BPE) tokenizer with a vocabulary size of $|\mathcal{V}|{=}8192$, including padding and unknown tokens. 
This yields approximately $136\mathrm{M}$/$285\mathrm{K}$ train/validation tokens for \texttt{wikitext-103} and $105\mathrm{M}$/$240\mathrm{K}$ train/validation
tokens for \texttt{simplebooks-92}.

We train a causal GPT with rotary positional embeddings \citep{su2021roformer}: $d{=}384$, 6 attention heads, 12 transformer blocks, and context length $T{=}512$.
This architecture has ${\sim}27.5\mathrm{M}$ transformer parameters. 
It matches the hierarchical view in Section~\ref{sec:deep_hierarchy}: we pair a deep observation model (the GPT) with a single embedding-level attention-prior stage, and use the standard-conditioning localization results (Proposition~\ref{prop:no_logdet_deep}, Corollary~\ref{cor:localize_first}) to justify focusing the nontrivial stability correction at the embedding level.

The \texttt{EmbeddingPrior} is a single strictly causal attention layer
parameterized by a learnable $d{\times}d$ matrix $A$ that computes the
attention-weighted covariance $\Sigma_t$ and the log-Jacobian margin term
$-\sum_t \log|\det(I_d-\Sigma_t A)|$
(Theorem~\ref{thm:jacobian_dispersion}). We use the \emph{margin-only} variant:
\[
\mathcal{L} = \mathcal{L}_{\mathrm{CE}} + \lambda_m \cdot
\mathcal{L}_{\mathrm{margin}},
\]
after dropping the quadratic residual term from the training objective, as $L_2$ regularization is unnecessary given the explicit regularization provided by AdamW's weight decay and early stopping.

For proper comparison, the cross-entropy-only and margin-regularized runs use identical architectures initialized with the same random seed.
Optimizer: AdamW ($\mathrm{lr}{=}3\times10^{-4}$, weight decay $0.1$); cosine learning-rate schedule; 20 epochs; batch size 32; gradient clipping at norm 1.0. 
All experiments are performed on an H200 GPU.

Because both datasets are subword-tokenized, we report cross-entropy (in nats)
and token-level perplexity as the primary predictive-quality metrics. For
noise-robustness experiments, we also report the ratio of noisy to clean
perplexity.

\subsection{Training: cross-entropy-only vs.\ margin-regularized (Figure~\ref{fig:training_curves})}
\label{sec:exp_training}

We first demonstrate that the stability/margin term implied by the change-of-variables correction (Section~\ref{sec:single_layer_margin}) can be used as a drop-in auxiliary loss on real language data with only a modest change in predictive quality. 
This is a strictly controlled comparison: architecture, optimizer, and seed are identical; the only difference is whether the stability (log-barrier) term is included in the objective (see Section~\ref{sec:training_objective} for the corresponding optimization view). 
In this subsection, we focus on the cross-entropy-only baseline ($\lambda_m=0$) and a representative, well-performing nonzero margin setting ($\lambda_m=0.05$); a broader comparison of margin-penalty settings is reported later in Section~\ref{sec:exp_lambda}.

\begin{figure}[H]
\centering
\includegraphics[width=\linewidth]{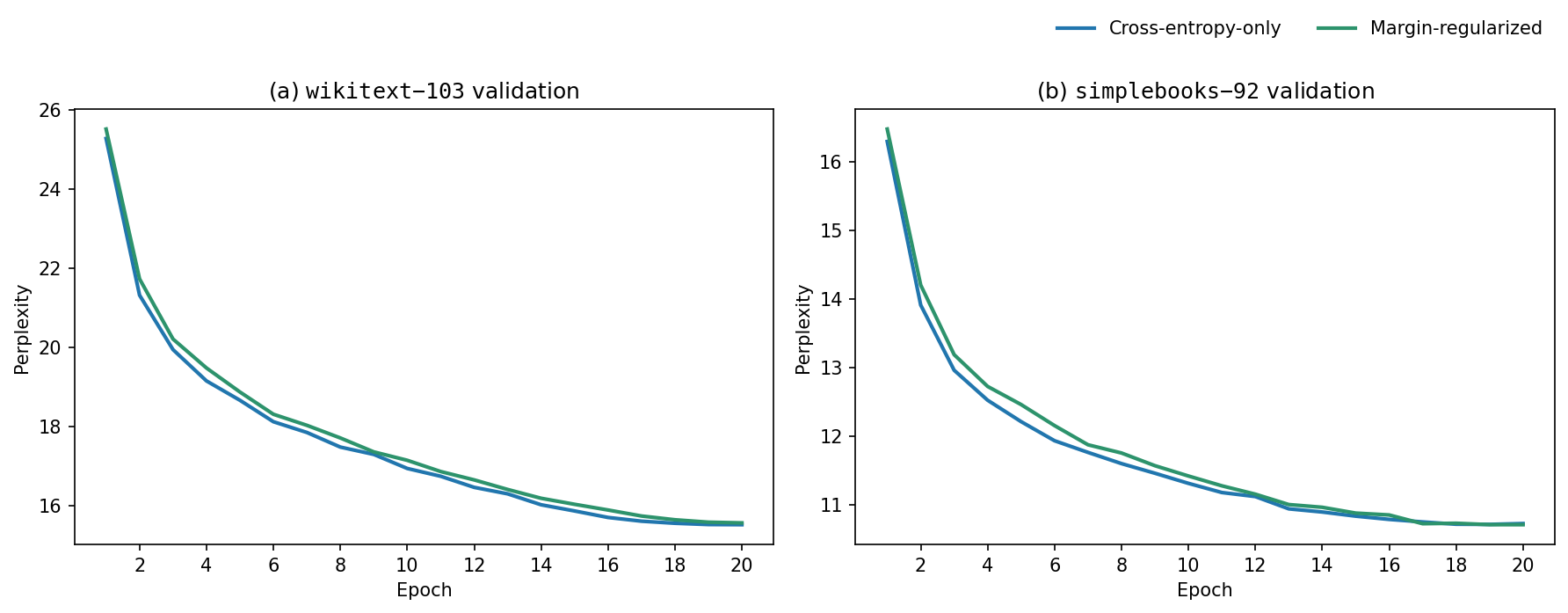}
\caption{\textbf{Training curves for cross-entropy-only vs.\ margin-regularized models.}
Perplexity over 20 epochs for the cross-entropy-only baseline (blue) and the margin-regularized model (green, $\lambda_m{=}0.05$).
From left to right: \texttt{wikitext-103} validation then \texttt{simplebooks-92} validation.
Across both datasets, the validation trajectories remain very close, indicating that the theory-derived margin penalty preserves predictive quality while adding a stability-oriented training signal.
}
\label{fig:training_curves}
\end{figure}

We train two GPT models on each of \texttt{wikitext-103} and \texttt{simplebooks-92} as described in Section~\ref{sec:exp_setup}, and report validation perplexity at every epoch. Both models converge smoothly with nearly overlapping validation curves on both datasets (Figure~\ref{fig:training_curves}).
The final training/validation perplexities are:

\begin{center}
\begin{tabular}{llcc}
\toprule
\textbf{Dataset} & \textbf{Objective} & \textbf{Training} & \textbf{Validation} \\
\midrule
\texttt{wikitext-103} & Cross-entropy-only & 12.31 & 15.53 \\
\texttt{wikitext-103} & Margin-regularized & 12.55 & 15.57 \\
\texttt{simplebooks-92} & Cross-entropy-only & 12.84 & 10.73 \\
\texttt{simplebooks-92} & Margin-regularized & 13.19 & 10.71 \\
\bottomrule
\end{tabular}
\end{center}

Across both datasets, the $\lambda_m=0.05$ margin regularizer leaves clean predictive performance essentially unchanged. 
On \texttt{wikitext-103}, the margin-regularized model finishes with slightly higher validation perplexity (15.57 vs.\ 15.53), while on \texttt{simplebooks-92} it is slightly better (10.71 vs.\ 10.73).
These small gaps support the interpretation of the margin term as a mild regularizer rather than a competing objective: it shapes the geometry of the learned representations without materially degrading the observation-model fit.

\subsection{Robustness to input perturbations (Figure~\ref{fig:robustness})}
\label{sec:exp_robustness}

Next, we test whether the stability/margin term improves robustness to perturbations in the input representation.
Our theory establishes that the exact embedding log-prior includes a log-volume stability factor that heavily penalizes near-degenerate configurations (Section~\ref{sec:single_layer_margin}). From an optimization perspective, this acts as a margin/log-barrier that discourages locally ill-conditioned geometry (Section~\ref{sec:training_objective}).
We therefore expect models trained with the stability term to degrade more gracefully under embedding-space corruption.

To test that prediction, we evaluate robustness under two simple
embedding-space perturbation families: unstructured i.i.d.\ noise and
structured drift \citep{sahoo2026can}. In our Gaussian setting, this
translates to:
\begin{enumerate}
\item \textbf{Unstructured Gaussian noise}: each token embedding is independently perturbed as
      \[
      x_t' = x_t + \varepsilon_t,
      \qquad
      \varepsilon_t \sim \mathcal{N}(0,\sigma^2 I),
      \]
      with $\sigma \in \{0, 0.05, 0.10, 0.15, 0.20, 0.25\}$. This probes local embedding corruption.
\item \textbf{Structured Gaussian drift}: each token embedding is perturbed along a
      shared random direction with token-specific Gaussian magnitudes,
      \[
      x_t' = x_t + \sigma a_t v,
      \qquad
      v \sim \mathcal{N}(0, I_d), \ \|v\|_2 = 1,
      \qquad
      a_t \sim \mathcal{N}(0,1),
      \]
      with $\sigma \in \{0, 1.5, 3.0, 4.5, 6.0\}$. This probes robustness to a
      coherent sequence-level shift along a shared random embedding direction.
\end{enumerate}
Under each perturbation, we evaluate the cross-entropy-only baseline and the
margin-regularized model with $\lambda_m{=}0.05$ on repeated
validation-window subsamples, using 100 subsamples of 16 validation windows
per condition. In the main figure, we report relative degradation
(perturbed perplexity / clean perplexity) over \texttt{wikitext-103} as boxplots over the subsamples.
Here $\sigma$ is a family-specific scale parameter, so its numerical values are
not directly comparable across the two perturbation families.
The corresponding \texttt{simplebooks-92} figure and the full numeric results
for both datasets are reported in Appendix~\ref{app:appendix_robustness}.

\begin{figure}[H]
\centering
\includegraphics[width=\linewidth]{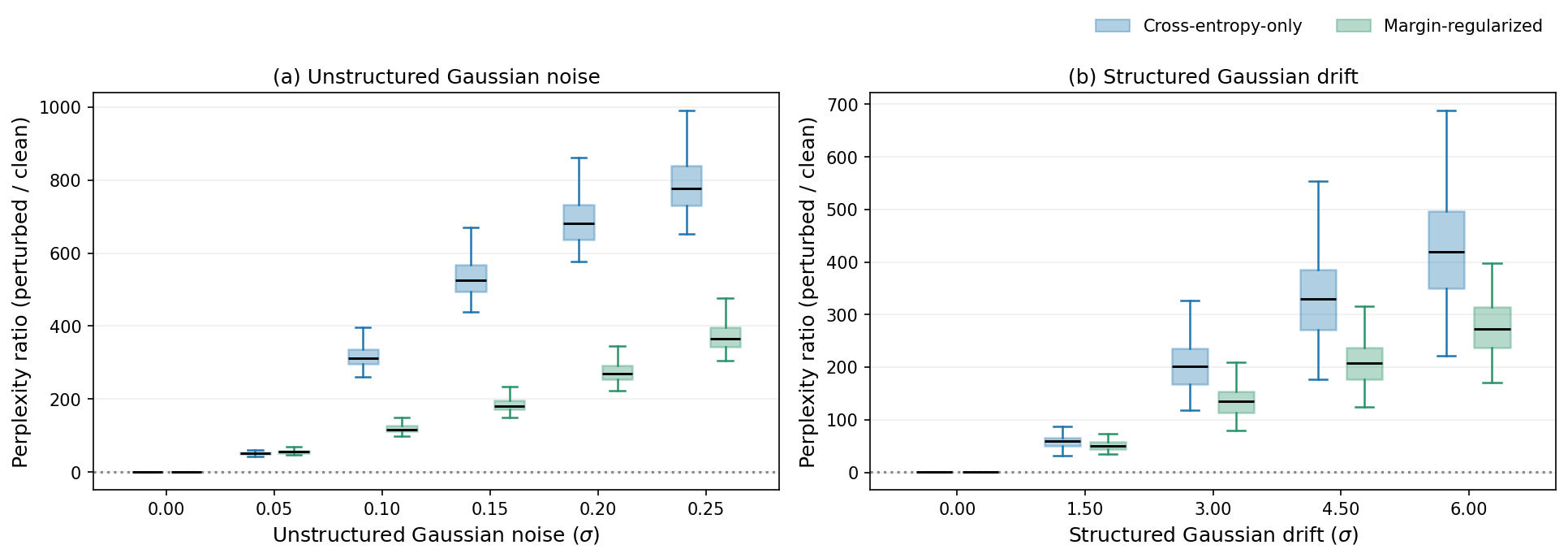}
\caption{\textbf{Robustness on \texttt{wikitext-103}: relative degradation under unstructured Gaussian noise and structured Gaussian drift.}
(a)~Unstructured Gaussian noise: grouped boxplots of perturbed perplexity / clean perplexity across 100 validation subsamples for $\sigma \in \{0, 0.05, 0.10, 0.15, 0.20, 0.25\}$. (b)~Structured Gaussian drift: grouped boxplots for $\sigma \in \{0, 1.5, 3.0, 4.5, 6.0\}$. In both panels, each box summarizes the empirical distribution across validation subsamples.}
\label{fig:robustness}
\end{figure}

Figure~\ref{fig:robustness} shows the same qualitative pattern under both
perturbation families on \texttt{wikitext-103}. Under unstructured
noise, the margin-regularized model becomes increasingly more robust as the
perturbation strength grows: the gap in relative degradation widens steadily
across the tested noise levels. Under structured drift, the same
monotone separation appears again, with the margin-regularized model remaining
below the baseline at every nonzero drift level. This is exactly the pattern
predicted by the stability-margin view: the regularizer keeps the learned
embedding geometry farther from locally ill-conditioned configurations, and
that safety buffer translates into graceful degradation under
unstructured noise and coherent embedding drift.

\begin{center}
\begin{tabular}{llcc}
\toprule
\textbf{Dataset} & \textbf{Objective} & \textbf{Ratio (Noise, $\sigma=0.25$)} & \textbf{Ratio (Drift, $\sigma=6.0$)} \\
\midrule
\texttt{wikitext-103} & Cross-entropy-only & 775.81 & 419.11 \\
\texttt{wikitext-103} & Margin-regularized & 364.92 & 273.24 \\
\texttt{simplebooks-92} & Cross-entropy-only & 413.45 & 250.71 \\
\texttt{simplebooks-92} & Margin-regularized & 130.91 & 107.96 \\
\bottomrule
\end{tabular}
\end{center}

The endpoint perplexity ratios in the table show that the same pattern holds across both
datasets. At $\sigma{=}0.25$, the unstructured noise degradation ratio drops from
$775.81{\times}$ to $364.92{\times}$ on \texttt{wikitext-103} and from
$413.45{\times}$ to $130.91{\times}$ on \texttt{simplebooks-92}. At the
largest structured-drift level $\sigma{=}6.0$, the degradation ratio drops from
$419.11{\times}$ to $273.24{\times}$ on \texttt{wikitext-103} and from
$250.71{\times}$ to $107.96{\times}$ on \texttt{simplebooks-92}. Together,
these experiments demonstrate strong empirical support for robustness against
both unstructured noise and structured drift in embedding
space, directly supporting the stability-margin interpretation of the learned
prior.

\subsection{Support tokens and margin concentration (Figure~\ref{fig:support_tokens})}
\label{sec:exp_support_tokens}

We next examine our support-token interpretation more directly.
Section~\ref{sec:single_layer_margin} defines support tokens as the positions closest to the degeneracy boundary, i.e., the positions that most strongly constrain the sequence-level margin.
In the learned \texttt{EmbeddingPrior} used here, our earlier result (Theorem~\ref{thm:jacobian_dispersion}) shows that the residual Jacobian block at position $t$ is $I_d-\Sigma_t A$,
so the corresponding token-wise contribution to the log-barrier term from \eqref{eq:matrix_prior_objective} is
\[
b_t \;=\; -\log \left|\det(I_d-\Sigma_t A)\right|.
\]
Large values of $b_t$ identify positions whose attention geometry lies closest to the degeneracy boundary and therefore most strongly constrains the sequence-level margin.
To compare how that exact barrier pressure is distributed across positions, we convert the barrier scores into a normalized distribution over positions:
\[
\pi_t \propto \exp(b_t).
\]
The cumulative curve then reports how much of the total barrier
pressure is captured by the top-$k$ ranked tokens. We compare the learned
margin-regularized prior at a representative, well-performing nonzero setting
($\lambda_m{=}0.05$ on both datasets) against a randomly initialized baseline
evaluated on the same embeddings. The random-init baseline uses the same
embeddings and the same barrier-based construction, but removes any learned geometry
from the barrier term.

\begin{figure}[H]
\centering
\includegraphics[width=\linewidth]{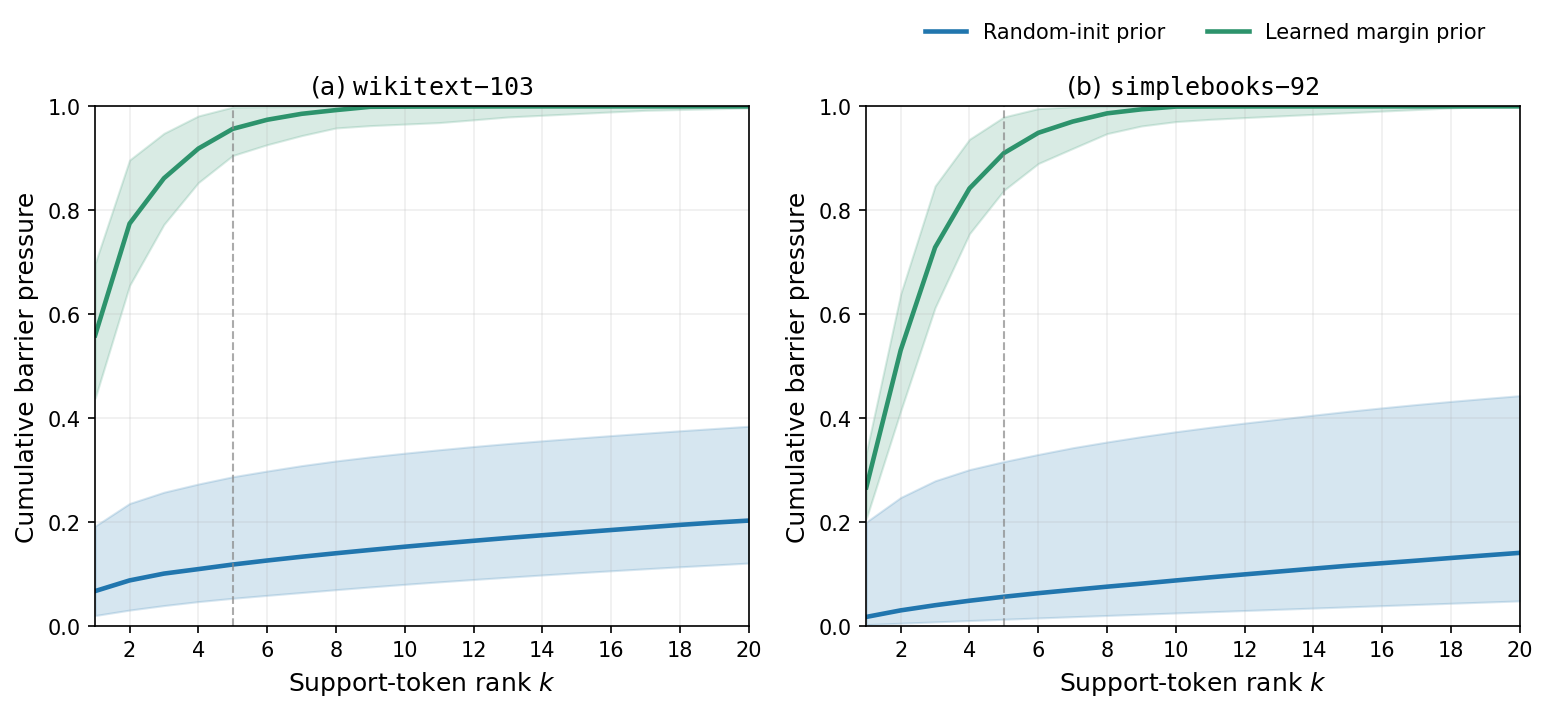}
\caption{\textbf{Support-token concentration under the learned margin prior and a random-initialized baseline.}
For each dataset, we compute the exact per-position barrier score
$b_t$, convert it into a normalized distribution over positions, sort token
positions by barrier weight, and plot the cumulative barrier pressure carried
by the top-$k$ support tokens. Solid lines show the median across validation subsamples and
shaded bands show the empirical 95\% interval. Left:
\texttt{wikitext-103}. Right: \texttt{simplebooks-92}. The dashed line marks
$k{=}5$.}
\label{fig:support_tokens}
\end{figure}

Figure~\ref{fig:support_tokens} shows a clear separation between the learned
and random baselines.
Under the learned margin prior, the cumulative curves rise sharply,
indicating that the exact Jacobian-barrier pressure is concentrated in a very
small set of ranked support tokens. By contrast, the random-init baseline
spreads barrier pressure broadly across many positions. This is exactly the
pattern we would expect if the support-token structure were a learned
geometric property rather than an artifact of the normalization itself.

\begin{center}
\begin{tabular}{llcc}
\toprule
\textbf{Dataset} & \textbf{Prior} & \textbf{Top-5 share} & \textbf{Effective support size} \\
\midrule
\texttt{wikitext-103} & Learned margin & 0.957 & 3.10 \\
\texttt{wikitext-103} & Random-init & 0.118 & 382.08 \\
\addlinespace[2pt]
\texttt{simplebooks-92} & Learned margin & 0.910 & 4.88 \\
\texttt{simplebooks-92} & Random-init & 0.056 & 443.89 \\
\bottomrule
\end{tabular}
\end{center}

The table above summarizes the same margin concentration pattern numerically.
Writing the sorted normalized barrier weights as
$\pi_{(1)} \ge \pi_{(2)} \ge \cdots$, we report two summaries:
\begin{enumerate}
\item \textbf{Top-5 share}:
      \[
      \sum_{i=1}^5 \pi_{(i)}.
      \]
      This measures how much of the total barrier pressure is carried by the
      five token positions with largest barrier weight.
\item \textbf{Effective support size}:
      \[
      \exp(H(\pi)),
      \qquad
      H(\pi) = -\sum_t \pi_t \log \pi_t .
      \]
      This measures how many token positions carry that barrier weight in
      aggregate; it is the exponential of the entropy of the normalized
      barrier-weight distribution, so it behaves like a count of how many
      token positions meaningfully contribute.
\end{enumerate}

Intuitively, a large top-5 share means that the sequence-level margin is
controlled by a small set of support tokens. Likewise, a small effective support size
means that a few token positions carry most of the barrier pressure. On
\texttt{wikitext-103}, the learned margin prior is extremely concentrated: the
effective support size is only $3.1$ tokens (95\% interval $[2.13,~4.02]$), and the top-$5$ tokens account for over $95\%$ of the total
barrier pressure (95\% interval $[90.5\%,~99.9\%]$). The same story holds on \texttt{simplebooks-92}: the effective
support size remains small at $4.9$ tokens (95\% interval $[3.83,~6.15]$), while
the top-$5$ tokens capture over $91\%$ of the total barrier pressure (95\% interval $[83.8\%,~98.0\%]$). By
contrast, the random-init baseline remains highly diffuse on both datasets,
with effective support sizes in the hundreds of tokens and only a small
fraction of the barrier pressure captured by the top five positions.

If the sharp concentration of barrier pressure across a few token positions
were merely an artifact of the normalization procedure, then applying the same
barrier-based construction to an untrained prior would also produce sharply
concentrated curves. Instead, under the random-init baseline, the barrier
pressure spreads broadly across token positions. The learned
\texttt{EmbeddingPrior} therefore organizes the exact Jacobian-barrier
pressure around a small set of geometrically critical support tokens,
precisely as the support-token interpretation predicts.

\subsection{Choosing the margin penalty (Figure~\ref{fig:lambda_sweep})}
\label{sec:exp_lambda}

We close with a direct model-selection exercise over a range of possible margin
settings
\[
\lambda_m \in \{0.00, 0.02, 0.05, 0.10, 0.25\}.
\]
The theory (Section~\ref{sec:training_objective}) predicts a trade-off:
increasing $\lambda_m$ should improve stability against perturbations, but an
overly large penalty can distort the learned embedding geometry and eventually
erode predictive quality. To measure that trade-off, we train one model for each
candidate margin-penalty setting, evaluate each model on repeated validation
subsamples, and summarize four quantities:
\begin{enumerate}
\item \textbf{Clean perplexity}, as the predictive-quality metric.
\item \textbf{Degradation ratio} at fixed unstructured Gaussian noise
      $\sigma{=}0.25$, defined as perturbed perplexity divided by clean
      perplexity.
\item \textbf{Contextual embedding error}, a model-side metric motivated
      by the residual map $e_t(x)=x_t-\mu_t(x)$ from
      Section~\ref{sec:graphical_model} and by the squared-error fit term in
      \eqref{eq:matrix_prior_objective}. Concretely, it is the normalized
      squared distance between the true token embedding and the embedding
      predicted from the model's next-token distribution given the preceding
      context. Smaller values mean the token embedding is more tightly
      predicted by context; larger values mean more embedding-space variation remains unexplained.
\item \textbf{Effective support size}, i.e., the entropy-based effective number of support tokens under the learned prior. This metric requires a trained prior, so it is only reported for nonzero margin-penalty settings. For the no-prior baseline, see the random-initialized comparison in Section~\ref{sec:exp_support_tokens}.
\end{enumerate}

\begin{figure}[H]
\centering
\includegraphics[width=\linewidth]{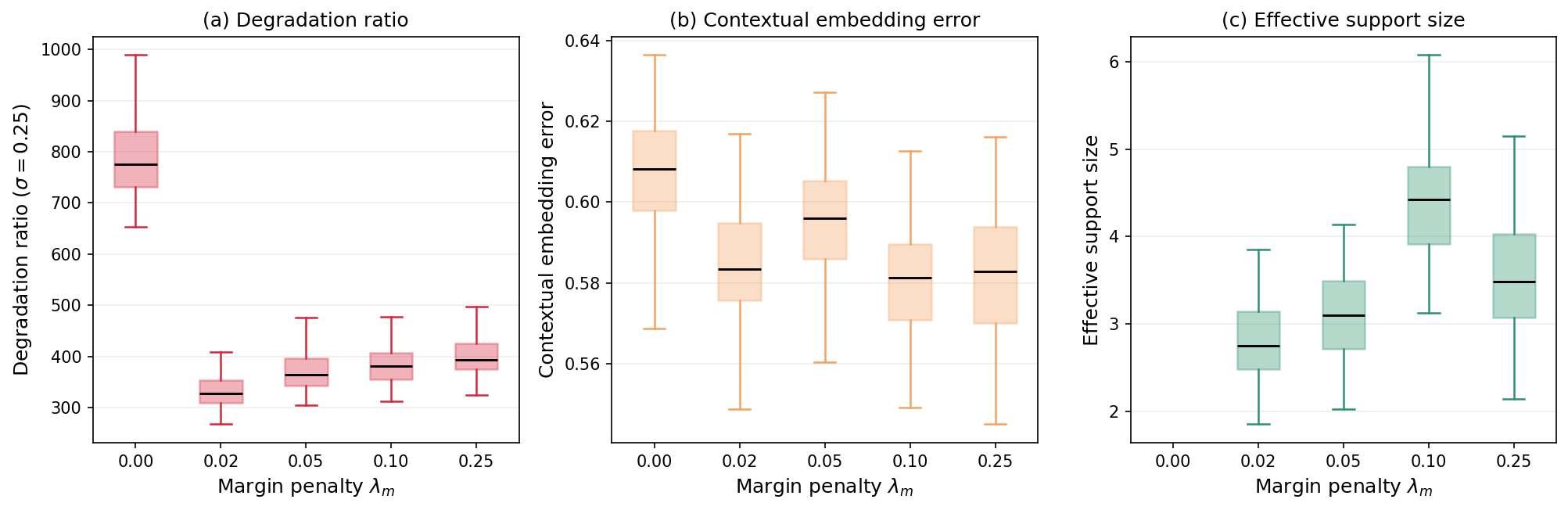}
\caption{\textbf{Selecting $\lambda_m$ on \texttt{wikitext-103}.}
Grouped boxplots over 100 validation subsamples for
$\lambda_m \in \{0.00, 0.02, 0.05, 0.10, 0.25\}$.
(a)~Degradation ratio at fixed unstructured Gaussian noise $\sigma{=}0.25$.
(b)~Contextual embedding error.
(c)~Effective support size.}
\label{fig:lambda_sweep}
\end{figure}

\begin{table}[H]
\centering
\caption{Regularization metrics at the three most relevant margin settings.
Full two-dataset results for all five trained values are reported in
Appendix~\ref{app:appendix_lambda}.}
\label{tab:lambda_sweep}
\small
\begin{tabular}{llcccc}
\toprule
\shortstack[c]{\textbf{Dataset}\\\mbox{}} & \shortstack[c]{\textbf{$\lambda_m$}\\\mbox{}} & \shortstack[c]{\strut\textbf{Clean}\\\textbf{perplexity}\strut} & \shortstack[c]{\strut\textbf{Degradation}\\\textbf{ratio}\strut} & \shortstack[c]{\strut\textbf{Contextual}\\\textbf{embedding error}\strut} & \shortstack[c]{\strut\textbf{Effective}\\\textbf{support size}\strut} \\
\midrule
\texttt{wikitext-103} & 0.00 & 15.68 & 775.81 & 0.608 & --- \\
\texttt{wikitext-103} & 0.02 & 15.74 & 328.51 & 0.583 & 2.75 \\
\texttt{wikitext-103} & 0.05 & 15.71 & 364.92 & 0.596 & 3.10 \\
\addlinespace[2pt]
\texttt{simplebooks-92} & 0.00 & 10.67 & 413.45 & 0.626 & --- \\
\texttt{simplebooks-92} & 0.02 & 10.75 & 127.30 & 0.554 & 4.77 \\
\texttt{simplebooks-92} & 0.05 & 10.67 & 130.91 & 0.557 & 4.88 \\
\bottomrule
\end{tabular}
\end{table}

\paragraph{(a) Predictive quality remains nearly flat.}
Table~\ref{tab:lambda_sweep} shows that clean validation perplexity changes
very little across the low nonzero range. On \texttt{wikitext-103}, the three
settings $\lambda_m \in \{0.00, 0.02, 0.05\}$ differ by less than
$0.07$ perplexity. On \texttt{simplebooks-92}, $\lambda_m{=}0.05$ is slightly
better on clean perplexity, but the gap relative to $\lambda_m{=}0.02$ is still small.
Thus, the margin penalty behaves as a mild regularizer rather than a competing objective.

\paragraph{(b) Robustness improves sharply at the first nonzero setting.}
The strongest signal in both datasets is the degradation ratio. On
\texttt{wikitext-103}, the median ratio drops from $775.81{\times}$ at
$\lambda_m{=}0.00$ to $328.51{\times}$ at $\lambda_m{=}0.02$, then rises
gradually to $364.92{\times}$ at $\lambda_m{=}0.05$ and continues upward for
larger penalties. The same pattern appears on \texttt{simplebooks-92}, where
the median ratio drops from $413.45{\times}$ at $\lambda_m{=}0.00$ to
$127.30{\times}$ at $\lambda_m{=}0.02$, with slightly worse values at
$\lambda_m{=}0.05$, $0.10$, and $0.25$. The main practical takeaway is
therefore not a broad optimum at a large margin penalty, but a sharp robustness gain
as soon as the margin term is turned on.

\paragraph{(c) Geometry supports the low-nonzero regime.}
The two geometric metrics in Figure~\ref{fig:lambda_sweep} play different
roles. \textbf{Contextual embedding error} asks how well the current token
embedding is predicted by the preceding context. It is substantially lower for every
nonzero margin setting than for $\lambda_m{=}0.00$ on both datasets, meaning
that the token embeddings produced by the margin-trained models are more
consistent with the embeddings predicted from preceding context.
Across the nonzero settings, however, its empirical intervals overlap, so it
supports using a positive margin penalty without sharply distinguishing among
the low nonzero values. \textbf{Effective support size} asks how many token
positions carry the barrier weight in aggregate. On both datasets, the learned
prior places that weight on only a handful of token positions for every
nonzero margin setting, with the strongest concentrations generally occurring
at the lower nonzero penalties, although the empirical intervals still
overlap. Taken together, these geometric metrics
support the same practical operating range as the predictive metrics:
$\lambda_m \in [0.02, 0.05]$, with $\lambda_m{=}0.02$ the stronger robustness
choice and $\lambda_m{=}0.05$ remaining competitive on clean perplexity.

\paragraph{Connection to the margin analogy.}
These results provide empirical support for the large-margin interpretation from
Section~\ref{sec:training_objective}. The role of $\lambda_m$ is to control how strongly the model
is pushed away from the Jacobian degeneracy boundary $\det(I_d-\Sigma_t A)=0$.
The explicit comparison of margin-penalty settings shows that $\lambda_m$ is
not a ``more is better'' knob: the first nonzero settings deliver the main
robustness gain, while larger penalties offer no further improvements to robustness or geometry,
and can gradually erode that benefit. This is the same qualitative role played
by a margin-control parameter in classical large-margin methods: one does not
choose the strongest possible penalty, but the smallest penalty that creates a
useful stability buffer without paying unnecessary predictive cost.

In our setting, this exercise identifies the low nonzero range $\lambda_m \in [0.02, 0.05]$
as the empirically selected margin-penalty regime. That range widens the
effective margin away from the degeneracy boundary enough to sharply improve
perturbation robustness and improve the geometry of the learned prior, while
still maintaining nearly unchanged clean perplexity.

\section{Related Work}
\label{sec:related_work}

Our contribution connects four broad lines of research: (i) mechanisms and interpretations of
self-attention in transformers, (ii) probabilistic modeling via latent variables and
change-of-variables objectives, (iii) stability, margins, and regularization viewpoints in
modern ML, and (iv) stochastic process foundations for sequence modeling. We highlight the closest connections and clarify what is distinct about our framework.

Transformers introduced self-attention as a computational primitive for sequence transduction
and autoregressive generation \citep{VaswaniShPaUsJoGoKaPo17}.
A large empirical and theoretical literature studies what attention computes and how it shapes
context use, including work on locality, sparsity, and inductive biases of attention patterns. 
Our focus is complementary: rather than proposing a new attention mechanism, we ask what
\emph{probabilistic object} is induced when causal attention is treated as a token-dependent
generative rule over continuous hidden-state embeddings. This viewpoint makes explicit an
additional, model-implied term in the embedding likelihood that can be interpreted as a
stability/margin effect, and it provides an optimization lens that is directly actionable in
training.

Normalizing flows formalize exact-likelihood density modeling by composing invertible maps and
accounting for local volume change via the change-of-variables formula
\citep{PapamakariosNaReMoLa21}. A central design principle is to choose transformations with
tractable Jacobian determinants, e.g., coupling layers \citep{DinhSoBe17} and autoregressive
constructions such as IAF/MAF \citep{KingmaSaJoChSuWe16,PapamakariosPaMu17}.
Unlike the flow literature, we do not design a new invertible architecture. Instead, we observe
that \emph{standard causal self-attention}, viewed as a latent-noise transformation on
embeddings, already induces an exact embedding likelihood with a nontrivial local-volume term.
Our main technical results compute this term for causal attention and interpret it geometrically
through a margin-to-degeneracy / log-barrier perspective.

A classical lesson from latent-variable modeling is that deterministic estimators often arise as
limits of probabilistic models. A canonical example is probabilistic PCA, where standard PCA
is recovered in a limiting (small-noise) regime \citep{TippingBi99}. Our treatment of embeddings
as latent variables plays a similar conceptual role: the noise scale provides an optimization
lens that connects an explicit probabilistic embedding prior to familiar deterministic behavior,
while the stability term remains as a principled geometric correction in the exact likelihood.

Large-margin classifiers provide a geometric view of learning in which a small set of
constraints (support vectors) governs robustness and generalization \citep{CortesVa95}.
Separately, interior-point methods use log-barriers as smooth relaxations of hard constraints,
offering a principled way to trade off objective fit against constraint satisfaction.
Our contribution is not to import these tools directly, but to show that an analogous barrier-like
mechanism emerges \emph{from the exact likelihood} induced by token-dependent causal attention:
the additional term behaves as a stability factor and naturally defines a margin to a critical
degeneracy boundary. This yields a practical training lever that can be tuned analogously to a
regularization path.

To connect embedding priors to token modeling at the dataset level, it is important that the
resulting family of finite-dimensional token distributions is consistent across sequence lengths.
This is the classical setting of projective consistency and extension to a stochastic process
(Kolmogorov-style) \citep{Billingsley95, Kolmogorov33}. In our setting, causality (masking) is the key structural
ingredient that enables such consistency, while non-causal attention can break it. We treat this as a
technical foundation for learning from variable-length sequences under an explicit probabilistic model.

Ultimately, by synthesizing these geometric and probabilistic perspectives, our framework moves beyond heuristic regularization. By exposing the intrinsic mathematical structure of standard transformers, we unlock principled, architecture-preserving strategies for training, decoding, and uncertainty quantification.

\section{Discussion and Future Directions}
\label{sec:discussion}

Our main contribution is conceptual and structural: by treating hidden-state embeddings as latent random variables and causal self-attention as a token-dependent latent-noise transformation, we obtain an \emph{exact} embedding likelihood whose additional change-of-variables term induces a margin-to-degeneracy viewpoint.
This re-frames causal self-attention not merely as a context aggregator \citep{VaswaniShPaUsJoGoKaPo17}, but as a generative mechanism with an intrinsic stability geometry.

The resulting barrier is best understood as a conditioning constraint: it discourages trajectories that approach locally singular configurations of the attention-induced map.
In this sense, it is closer in spirit to classical stability constraints than to generic ``make attention sparse'' regularizers.
While the sign and strength of the effective coupling determine how dispersion influences local sensitivity, the degeneracy boundary and margin interpretation remain the invariant object: the likelihood assigns vanishing mass to configurations approaching singularity.

Crucially, this barrier term is \emph{model-implied} (not heuristic) and can be surfaced in training as a lightweight penalty without changing the core transformer architecture. 
While our primary focus is theoretical, our initial experiments on language models validate this framework. By optimizing the proposed margin penalty, we validate the emergence of distinct geometric structures in the latent space and observe measurable improvements in model robustness, confirming that this mathematically derived objective translates into practical behavioral shifts.

\paragraph{Posterior-aware decoding beyond MAP.}
The latent-noise formulation defines a posterior over embedding trajectories (and, in principle, intermediate layer states) conditioned on an observed prefix.
This opens the door to decoding rules that go beyond a single point-estimate (MAP) embedding path.
Conceptually, the next-token distribution can be written as an integral,
\[
P(y_t \mid y_{<t}) \;=\; \int P(y_t \mid y_{<t}, x_{1:t})\, p(x_{1:t} \mid y_{<t})\, dx_{1:t},
\]
and approximated by sampling or structured inference rather than collapsing $x$ to one trajectory.
Even coarse approximations that preserve the main geometry (staying away from near-degenerate configurations) may yield practical gains in calibration and controllability.

\paragraph{Uncertainty as a first-class output.}
Because the model assigns an explicit density to embedding trajectories, it becomes meaningful to report uncertainty measures tied directly to the latent representation.
Natural candidates include posterior dispersion in embedding space, predictive variance under posterior samples, or proximity to degeneracy via the margin/log-barrier quantities derived in Sections~\ref{sec:single_layer_margin}--\ref{sec:training_objective}.
Such signals could be used to (i) flag low-confidence continuations, (ii) calibrate selective generation (abstain, ask a clarifying question, or trigger retrieval) in the spirit of selective prediction \citep{GeifmanEl17}, or (iii) adapt decoding hyper-parameters (temperature, top-$p$) in a principled, state-dependent manner rather than via fixed heuristics.
This perspective is complementary to classical probability calibration work, which studies when predicted probabilities match empirical correctness \citep{GuoPlSuWe17}.

\paragraph{Sequential inference for long contexts.}
The causal structure invites sequential inference algorithms that update uncertainty as the prefix grows.
Promising directions include particle methods (SMC) \citep{DoucetDeGo01} and MCMC over latent embeddings conditioned on tokens.
Hybrid particle-MCMC approaches \citep{AndrieuDoHo10} are also natural candidates when the posterior is highly structured and naive proposals mix poorly.
In all cases, the barrier term can act as a principled prior factor that discourages proposals drifting into near-degenerate regions.
Approximate filtering-style methods may also be useful if they preserve the key geometry while remaining computationally feasible at long context lengths.

\paragraph{Robustness and hallucinations under distribution shift.}
A practical hypothesis is that uncertainty-aware decoding informed by the latent posterior could reduce hallucinations.
When the model enters regions of high posterior uncertainty or small margin (near-degenerate attention geometries), decoding could become more conservative (e.g., trigger retrieval/verification, lower entropy, ask for clarification, or refuse).
Testing this requires evaluations explicitly designed around distribution shift, factuality, and selective generation, beyond perplexity-style metrics; surveys of hallucination phenomena and mitigation provide useful taxonomies and benchmarks for such evaluation \citep{HuangYuMa23Etal}.

\paragraph{Scaling and implementation.}
While our initial experiments successfully demonstrate the geometric and robustness benefits of the margin penalty, scaling this framework introduces three main practical challenges.
The first is computational: enforcing the exact log-barrier is expensive, necessitating scalable techniques such as low-rank factorizations or stochastic log-determinant estimators.
The second involves training stability: integrating this geometric penalty into modern large-batch regimes requires careful tuning. Practical solutions include deploying numerical relaxations, barrier schedules, and constrained parameterizations, alongside efficiently managing the objective tradeoff via warm-start regularization paths.
Finally, establishing reliable large-scale wins will require rigorous ablations across model sizes, data regimes, and downstream tasks to clarify whether the strongest benefits come from training-time regularization, inference-time uncertainty-aware decoding, or alignment-style objectives built directly on these uncertainty signals.

\setlength{\bibsep}{2pt}
\bibliographystyle{abbrvnat} 
\bibliography{bib}

\appendix

\section{Extension to Non-Strict Causal Attention (\texorpdfstring{$s \le t$}{s <= t})}
\label{app:self_in_context}

The main text uses the \emph{strictly causal} attention summary $\mu_t=\sum_{s<t}\alpha_{ts}v_s$ to keep the autoregressive structure maximally transparent. Standard implementations of causal self-attention often include the current position as well (i.e., they allow $s\le t$). This appendix records the corresponding formulas and shows that the change-of-variables likelihood continues to acquire an additional log-determinant factor; the only difference is that the diagonal derivative includes an extra contribution from the explicit dependence of $v_t$ on $x_t$ and from the nonzero self-weight $\alpha_{tt}$.

\paragraph{Setup.}
Fix a single-head self-attention layer. Define
\[
q_t = W_Q x_t, \qquad k_s = W_K x_s, \qquad v_s = W_V x_s,
\]
and let the (masked) attention weights be
\begin{equation}
\alpha_{ts}(x)
\;=\;
\frac{\exp(q_t^\top k_s)}{\sum_{r\le t}\exp(q_t^\top k_r)},
\qquad s\le t,
\label{eq:attn_weights_self}
\end{equation}
with context summary
\begin{equation}
\mu_t(x) \;=\; \sum_{s\le t}\alpha_{ts}(x)\,v_s.
\label{eq:mu_self}
\end{equation}
As in the main text, we define the latent-noise residual map
\begin{equation}
\varepsilon_t \;=\; e_t(x) \;\triangleq\; x_t - \mu_t(x),
\qquad
\varepsilon_t \sim \mathcal{N}(0,\sigma^2 I_d).
\label{eq:residual_self}
\end{equation}

\paragraph{Useful identities.}
Write $\bar v_t \triangleq \sum_{s\le t}\alpha_{ts}v_s = \mu_t(x)$ for the attention-weighted mean of the values, and define the attention-weighted covariance over values (now including $s=t$):
\begin{equation}
\Sigma_t^{(\le)} \;\triangleq\; \sum_{s\le t}\alpha_{ts}(x)\,(v_s-\bar v_t)(v_s-\bar v_t)^\top \succeq 0.
\label{eq:cov_self}
\end{equation}
For the softmax weights \eqref{eq:attn_weights_self} with logits $\ell_{ts}=q_t^\top k_s$, the derivative w.r.t.\ the current query satisfies
\begin{equation}
\frac{\partial \alpha_{ts}}{\partial q_t}
\;=\;
\alpha_{ts}\left(k_s - \sum_{r\le t}\alpha_{tr}k_r\right).
\label{eq:softmax_q_deriv_self}
\end{equation}

\paragraph{Diagonal Jacobian block for $s\le t$.}
The diagonal block of the Jacobian of the residual map $e_t(x)=x_t-\mu_t(x)$ equals
\[
\frac{\partial e_t}{\partial x_t}
\;=\;
I_d \;-\; \frac{\partial \mu_t}{\partial x_t}.
\]
We decompose $\partial\mu_t/\partial x_t$ into two contributions:

\smallskip
\noindent
(i) \emph{Direct value-path term.} Since $v_t=W_Vx_t$ appears inside the sum with weight $\alpha_{tt}$,
\begin{equation}
\left.\frac{\partial \mu_t}{\partial x_t}\right|_{\text{value path}}
\;=\;
\alpha_{tt}\,W_V.
\label{eq:value_path_self}
\end{equation}

\smallskip
\noindent
(ii) \emph{Weight-path term via the query.} The weights $\alpha_{ts}$ depend on $x_t$ through $q_t=W_Qx_t$. Using
$\frac{\partial \alpha_{ts}}{\partial x_t}
=
\left(\frac{\partial \alpha_{ts}}{\partial q_t}\right) W_Q$
and \eqref{eq:softmax_q_deriv_self}, one obtains
\begin{align}
\left.\frac{\partial \mu_t}{\partial x_t}\right|_{\text{weight path}}
&=
\sum_{s\le t} v_s \left(\frac{\partial \alpha_{ts}}{\partial x_t}\right)
=
\sum_{s\le t} v_s \left(\frac{\partial \alpha_{ts}}{\partial q_t}\right) W_Q \nonumber\\
&=
\left(\sum_{s\le t}\alpha_{ts}\,v_s k_s^\top - \bar v_t \Big(\sum_{r\le t}\alpha_{tr}k_r\Big)^\top \right) W_Q \nonumber\\
&=
\left(\sum_{s\le t}\alpha_{ts}\,(v_s-\bar v_t)k_s^\top\right) W_Q.
\label{eq:weight_path_self}
\end{align}
In the common bilinear parameterization $k_s=W_Kx_s$ and $v_s=W_Vx_s$, the bracketed term in \eqref{eq:weight_path_self} can be rewritten in terms of the covariance \eqref{eq:cov_self}:
\begin{equation}
\sum_{s\le t}\alpha_{ts}\,(v_s-\bar v_t)k_s^\top
=
\Sigma_t^{(\le)}\,W_K^\top.
\label{eq:cov_rewrite_self}
\end{equation}
Substituting \eqref{eq:value_path_self}, \eqref{eq:weight_path_self}, and \eqref{eq:cov_rewrite_self} gives
\begin{equation}
\frac{\partial \mu_t}{\partial x_t}
\;=\;
\alpha_{tt} W_V \;+\; \Sigma_t^{(\le)}\,W_K^\top W_Q.
\label{eq:dmu_dxt_self}
\end{equation}
Therefore the diagonal Jacobian block becomes
\begin{equation}
\boxed{
\frac{\partial e_t}{\partial x_t}
\;=\;
I_d \;-\; \alpha_{tt} W_V \;-\; \Sigma_t^{(\le)}\,A,
\qquad
A \triangleq W_K^\top W_Q.
}
\label{eq:diag_block_self}
\end{equation}

\paragraph{Specialization to $W_V=I_d$.}
In the common simplified analysis where the values are taken as $v_s=x_s$ (equivalently $W_V=I_d$), \eqref{eq:diag_block_self} reduces to
\begin{equation}
\frac{\partial e_t}{\partial x_t}
\;=\;
(1-\alpha_{tt})\,I_d \;-\; \Sigma_t^{(\le)}\,A.
\label{eq:diag_block_self_WV_I}
\end{equation}
This makes explicit the additional $(1-\alpha_{tt})$ contraction that appears solely because the current token is included in the context.

\paragraph{Resulting log-determinant term.}
As in the strict-causal case, the full Jacobian is block lower-triangular in time under causality, hence
\begin{equation}
\det\!\left(J_{x\mapsto \varepsilon}\right)
\;=\;
\prod_{t=1}^L
\det\!\left(\frac{\partial e_t}{\partial x_t}\right).
\label{eq:full_det_self}
\end{equation}
Under the Gaussian base density, the induced log-density becomes
\begin{equation}
\log p(x_{1:L})
\;=\;
-\frac{1}{2\sigma^2}\sum_{t=1}^L \|e_t(x)\|_2^2
\;+\;
\sum_{t=1}^L \log\left|\det\!\left(I_d-\alpha_{tt}W_V-\Sigma_t^{(\le)}A\right)\right|
\;-\;
Ld\log\sigma
\;+\;
\mathrm{const},
\label{eq:loglik_self}
\end{equation}
with $e_t(x)=x_t-\mu_t(x)$ and $\mu_t$ defined by \eqref{eq:mu_self}. The second term is the analogue of the ``missing term'' in the strict-causal setting: it is again the exact change-of-variables correction induced by token-dependent attention.

\paragraph{Relation to the strict-causal case.}
If one excludes the current token ($s<t$), then $\alpha_{tt}=0$ and the covariance excludes $s=t$, so \eqref{eq:diag_block_self} reduces to the main-text diagonal block
\[
\frac{\partial e_t}{\partial x_t}
=
I_d - \Sigma_t^{(<)}A,
\]
recovering the formulas in Section~\ref{sec:single_layer_margin}. In particular, including $s=t$ does not remove the additional term in the likelihood; it only changes its exact form by adding the explicit self-value contribution $\alpha_{tt}W_V$ and replacing $\Sigma_t^{(<)}$ by $\Sigma_t^{(\le)}$.


\section{Proofs for Section~\ref{sec:single_layer_margin}}
\label{app:proofs_single_layer}

Throughout, we work in the \emph{strict-causal} setting $s<t$ as in
\eqref{eq:attn_summary_strict}--\eqref{eq:residual_map_graphical}, with
\[
\mu_t(x)=\sum_{s<t}\alpha_{ts}(x)\,v_s,\qquad e_t(x)=x_t-\mu_t(x),
\]
and single-head bilinear logits
\[
\ell_{ts}\;=\; q_t^\top k_s,\qquad q_t=W_Q x_t,\qquad k_s=W_K x_s.
\]
We write $\alpha_t$ for the vector $(\alpha_{t1},\dots,\alpha_{t,t-1})$ and use the standard softmax
$\alpha_{ts}=\exp(\ell_{ts})/\sum_{r<t}\exp(\ell_{tr})$.

\subsection{Two standard identities}
\label{app:lemmas_standard}

\begin{lemma}[Softmax Jacobian]
\label{lem:softmax_jacobian}
Let $\alpha=\mathrm{softmax}(\ell)\in\mathbb{R}^m$ with components
$\alpha_i=\exp(\ell_i)/\sum_{j=1}^m \exp(\ell_j)$. Then
\begin{equation}
\frac{\partial \alpha_i}{\partial \ell_j}
\;=\;
\alpha_i(\delta_{ij}-\alpha_j),
\qquad\text{i.e.,}\qquad
\frac{\partial \alpha}{\partial \ell}
\;=\;
\mathrm{Diag}(\alpha)-\alpha\alpha^\top.
\label{eq:softmax_jacobian}
\end{equation}
\end{lemma}

\begin{proof}
Differentiate $\alpha_i=\exp(\ell_i)/Z$ with $Z=\sum_j\exp(\ell_j)$.
We have $\partial \alpha_i/\partial \ell_j = \delta_{ij}\exp(\ell_i)/Z - \exp(\ell_i)\exp(\ell_j)/Z^2
= \alpha_i(\delta_{ij}-\alpha_j)$.
\end{proof}

\begin{lemma}[Attention-weighted covariance identity]
\label{lem:cov_identity}
Let $\{v_s\}_{s<t}\subset \mathbb{R}^d$ with weights $\alpha_{ts}\ge 0$, $\sum_{s<t}\alpha_{ts}=1$.
Define $\bar v_t=\sum_{s<t}\alpha_{ts}v_s$ and
\[
\Sigma_t=\sum_{s<t}\alpha_{ts}(v_s-\bar v_t)(v_s-\bar v_t)^\top.
\]
Then
\begin{equation}
\Sigma_t
=
\sum_{s<t}\alpha_{ts} v_s v_s^\top - \bar v_t \bar v_t^\top.
\label{eq:cov_identity}
\end{equation}
In the scalar case ($d=1$), this reduces to
$\mathrm{Var}_t=\sum_{s<t}\alpha_{ts}v_s^2-\bar v_t^2$.
\end{lemma}

\begin{proof}
Expand:
\[
\sum_{s<t}\alpha_{ts}(v_s-\bar v_t)(v_s-\bar v_t)^\top
=
\sum_{s<t}\alpha_{ts}v_s v_s^\top
-\bar v_t\sum_{s<t}\alpha_{ts}v_s^\top
-\Big(\sum_{s<t}\alpha_{ts}v_s\Big)\bar v_t^\top
+\Big(\sum_{s<t}\alpha_{ts}\Big)\bar v_t\bar v_t^\top.
\]
Use $\sum_{s<t}\alpha_{ts}v_s=\bar v_t$ and $\sum_{s<t}\alpha_{ts}=1$.
\end{proof}

\subsection{Block lower-triangular structure of the Jacobian}
\label{app:block_triangular}

\begin{lemma}[Causal residual map has block lower-triangular Jacobian]
\label{lem:block_lower_triangular}
Consider the map $x_{1:L}\mapsto e_{1:L}(x)$ with $e_t(x)=x_t-\mu_t(x)$ and
$\mu_t(x)$ depending only on $(x_1,\dots,x_t)$ (strict causality ensures it uses only $x_{<t}$
through the values, but it may depend on $x_t$ through the query).
Then the Jacobian $J=\partial e/\partial x$ is block lower-triangular:
\[
\frac{\partial e_t}{\partial x_s} = 0\qquad\text{for all } s>t.
\]
Consequently,
\begin{equation}
\det\!\left(\frac{\partial e}{\partial x}\right)
=
\prod_{t=1}^L \det\!\left(\frac{\partial e_t}{\partial x_t}\right).
\label{eq:block_det_product}
\end{equation}
\end{lemma}

\begin{proof}
For $s>t$, $\mu_t(x)$ does not depend on $x_s$ by causality, and $x_t$ obviously does not depend on $x_s$.
Hence $\partial e_t/\partial x_s=0$. A block lower-triangular matrix has determinant equal to the
product of determinants of its diagonal blocks, giving \eqref{eq:block_det_product}.
\end{proof}

\subsection{Proof of Proposition~\ref{prop:scalar_diag}}
\label{app:proof_prop_scalar}

\begin{proof}
Assume $d=1$. Then $W_Q,W_K,W_V$ are scalars; write them as $w_Q,w_K,w_V$.
The logits are
\[
\ell_{ts}=q_t k_s=(w_Q x_t)(w_K x_s)=(w_Q w_K)\,x_t x_s.
\]
Define $a\triangleq w_K w_Q$ (equivalently $a=W_KW_Q$ in scalar notation).
The attention weights are $\alpha_{ts}=\exp(\ell_{ts})/\sum_{r<t}\exp(\ell_{tr})$.
We take $v_s$ as in the main text (either $v_s=w_V x_s$ or a more general scalar $v_s$); the proof below
uses only that $v_s$ does not depend on $x_t$ for $s<t$.

We have
\[
e_t(x)=x_t-\mu_t(x),\qquad \mu_t(x)=\sum_{s<t}\alpha_{ts}(x)\,v_s.
\]
Hence
\[
\frac{\partial e_t}{\partial x_t} = 1-\frac{\partial \mu_t}{\partial x_t}.
\]
Differentiate $\mu_t$:
\[
\frac{\partial \mu_t}{\partial x_t}
=
\sum_{s<t} v_s \frac{\partial \alpha_{ts}}{\partial x_t}.
\]
By the chain rule,
\[
\frac{\partial \alpha_{ts}}{\partial x_t}
=
\sum_{m<t} \frac{\partial \alpha_{ts}}{\partial \ell_{tm}}
\frac{\partial \ell_{tm}}{\partial x_t}.
\]
Using Lemma~\ref{lem:softmax_jacobian},
$\partial \alpha_{ts}/\partial \ell_{tm}=\alpha_{ts}(\delta_{sm}-\alpha_{tm})$.
Also $\ell_{tm}=a\,x_t x_m$, so $\partial \ell_{tm}/\partial x_t = a\,x_m$.
Therefore
\begin{align}
\frac{\partial \alpha_{ts}}{\partial x_t}
&=
\sum_{m<t} \alpha_{ts}(\delta_{sm}-\alpha_{tm})\,a\,x_m
=
a\,\alpha_{ts}\Big(x_s-\sum_{m<t}\alpha_{tm}x_m\Big).
\end{align}
In the common specialization $v_s=x_s$ (or $v_s=w_V x_s$), it is natural to write the mean of the
attended values; to match the statement in the main text, define
\[
\bar v_t \triangleq \sum_{s<t}\alpha_{ts} v_s,
\qquad
\mathrm{Var}_t \triangleq \sum_{s<t}\alpha_{ts}(v_s-\bar v_t)^2.
\]
Now compute $\partial \mu_t/\partial x_t$:
\begin{align}
\frac{\partial \mu_t}{\partial x_t}
&=
\sum_{s<t} v_s \cdot a\,\alpha_{ts}\Big(v_s-\sum_{m<t}\alpha_{tm}v_m\Big)
\qquad\text{(identify $x_s$ with $v_s$ in the logit-coupling)} \nonumber\\
&=
a\sum_{s<t}\alpha_{ts}\,v_s(v_s-\bar v_t)
=
a\left(\sum_{s<t}\alpha_{ts}v_s^2-\bar v_t\sum_{s<t}\alpha_{ts}v_s\right)\nonumber\\
&=
a\left(\sum_{s<t}\alpha_{ts}v_s^2-\bar v_t^2\right)
=
a\,\mathrm{Var}_t,
\label{eq:scalar_mu_derivative}
\end{align}
where the last equality uses Lemma~\ref{lem:cov_identity} in $d=1$.
Finally,
\[
\frac{\partial e_t}{\partial x_t}
=
1-\frac{\partial \mu_t}{\partial x_t}
=
1-a\,\mathrm{Var}_t,
\]
which is \eqref{eq:scalar_diag_deriv}.
\end{proof}

\subsection{Proof of Theorem~\ref{thm:jacobian_dispersion}}
\label{app:proof_thm_matrix}

\begin{proof}
We now allow $d>1$. Recall
\[
\mu_t(x)=\sum_{s<t}\alpha_{ts}(x)\,v_s,\qquad e_t(x)=x_t-\mu_t(x),
\]
with $q_t=W_Q x_t$, $k_s=W_K x_s$, and logits $\ell_{ts}=q_t^\top k_s$.

We compute the diagonal Jacobian block $\partial e_t/\partial x_t$.
Since $e_t=x_t-\mu_t$, we have
\begin{equation}
\frac{\partial e_t}{\partial x_t}
=
I_d-\frac{\partial \mu_t}{\partial x_t}.
\label{eq:et_block}
\end{equation}
Because $v_s$ for $s<t$ does not depend on $x_t$, the only dependence of $\mu_t$ on $x_t$
is through the attention weights:
\[
\frac{\partial \mu_t}{\partial x_t}
=
\sum_{s<t} v_s \left(\frac{\partial \alpha_{ts}}{\partial x_t}\right)^\top,
\]
where $\partial \alpha_{ts}/\partial x_t\in\mathbb{R}^d$ and we write the result as a $d\times d$ matrix.

By the chain rule,
\[
\frac{\partial \alpha_{ts}}{\partial x_t}
=
\sum_{m<t}\frac{\partial \alpha_{ts}}{\partial \ell_{tm}}
\frac{\partial \ell_{tm}}{\partial x_t}.
\]
Lemma~\ref{lem:softmax_jacobian} gives
$\partial \alpha_{ts}/\partial \ell_{tm}=\alpha_{ts}(\delta_{sm}-\alpha_{tm})$.
Next,
\[
\ell_{tm}=q_t^\top k_m=(W_Q x_t)^\top (W_K x_m)=x_t^\top W_Q^\top W_K x_m,
\]
so
\[
\frac{\partial \ell_{tm}}{\partial x_t}
=
W_Q^\top W_K x_m.
\]
Define
\begin{equation}
A \triangleq W_K^\top W_Q \in \mathbb{R}^{d\times d}.
\label{eq:A_def_app}
\end{equation}
Then $W_Q^\top W_K = A^\top$, and we can write $\partial \ell_{tm}/\partial x_t = A^\top x_m$.

Therefore
\begin{align}
\frac{\partial \alpha_{ts}}{\partial x_t}
&=
\sum_{m<t}\alpha_{ts}(\delta_{sm}-\alpha_{tm})\,A^\top x_m
=
\alpha_{ts}\left(A^\top x_s-\sum_{m<t}\alpha_{tm}A^\top x_m\right)\nonumber\\
&=
\alpha_{ts}A^\top\left(x_s-\sum_{m<t}\alpha_{tm}x_m\right).
\label{eq:dalpha_dx}
\end{align}
Now specialize to the standard ``values as embeddings'' case used in the theorem statement:
$v_s=W_V x_s$. The attention-weighted mean of values is
\[
\bar v_t=\sum_{s<t}\alpha_{ts}v_s,
\]
and define the attention-weighted covariance of values,
\[
\Sigma_t=\sum_{s<t}\alpha_{ts}(v_s-\bar v_t)(v_s-\bar v_t)^\top.
\]
Using \eqref{eq:dalpha_dx},
\begin{align}
\frac{\partial \mu_t}{\partial x_t}
&=
\sum_{s<t} v_s\left(\frac{\partial \alpha_{ts}}{\partial x_t}\right)^\top
=
\sum_{s<t} v_s\left(\alpha_{ts}A^\top(x_s-\bar x_t)\right)^\top \nonumber\\
&=
\sum_{s<t} \alpha_{ts} v_s (x_s-\bar x_t)^\top A
\qquad\text{where }\bar x_t=\sum_{m<t}\alpha_{tm}x_m. 
\label{eq:mu_block_intermediate}
\end{align}
If we take the common choice $v_s=x_s$ (or more generally if the same vectors appear in logits and values
up to a linear map, absorbed into $\Sigma_t$), then $\bar v_t=\bar x_t$ and
\[
\sum_{s<t} \alpha_{ts} v_s (v_s-\bar v_t)^\top = \Sigma_t
\quad\text{(by Lemma~\ref{lem:cov_identity})}.
\]
Thus \eqref{eq:mu_block_intermediate} reduces to
\begin{equation}
\frac{\partial \mu_t}{\partial x_t}
=
\Sigma_t A.
\label{eq:mu_block_final}
\end{equation}
Substituting \eqref{eq:mu_block_final} into \eqref{eq:et_block} yields
\[
\frac{\partial e_t}{\partial x_t}
=
I_d-\Sigma_t A,
\]
which is exactly \eqref{eq:matrix_diag_block}. The identification $A=W_K^\top W_Q$ is \eqref{eq:A_basic}.
\end{proof}

\paragraph{Remark on $v_s=W_Vx_s$ vs.\ $v_s=x_s$.}
If $v_s=W_Vx_s$ and logits are computed from $x_s$ (keys) rather than from $v_s$, then the same derivation
goes through with $\Sigma_t$ replaced by the attention-weighted covariance of the \emph{values} that appear
inside $\mu_t$, i.e.,
$\Sigma_t=\sum_{s<t}\alpha_{ts}(v_s-\bar v_t)(v_s-\bar v_t)^\top$.
All statements in Section~\ref{sec:single_layer_margin} are intended in this ``covariance of attended values''
sense.

\subsection{Proof of Corollary~\ref{cor:logdet_penalty}}
\label{app:proof_cor_logdet}

\begin{proof}
Assume $\varepsilon_t\sim\mathcal{N}(0,\sigma^2 I_d)$ i.i.d., and define the residual map
$e_{1:L}(x)$ with components $e_t(x)=x_t-\mu_t(x)$.
Let $J(x)=\partial e/\partial x$ be the Jacobian of the map $x\mapsto e$.
When the map is locally invertible at $x$ (equivalently $\det J(x)\neq 0$), the change-of-variables
formula gives
\begin{equation}
\log p(x_{1:L})
=
\log p(\varepsilon_{1:L})
+
\log\left|\det J(x)\right|,
\qquad \varepsilon_{1:L}=e_{1:L}(x).
\label{eq:cov_general_app}
\end{equation}
The Gaussian base density yields
\[
\log p(\varepsilon_{1:L})
=
-\frac{1}{2\sigma^2}\sum_{t=1}^L \|\varepsilon_t\|_2^2
-\frac{Ld}{2}\log(2\pi\sigma^2).
\]
Substitute $\varepsilon_t=e_t(x)$ to obtain the data-fit term
$-\frac{1}{2\sigma^2}\sum_t\|e_t(x)\|_2^2$ plus constants.

Next, by Lemma~\ref{lem:block_lower_triangular}, $J(x)$ is block lower-triangular and
\[
\det J(x) = \prod_{t=1}^L \det\!\left(\frac{\partial e_t}{\partial x_t}\right).
\]
Taking logs gives
\[
\log|\det J(x)|=\sum_{t=1}^L \log\left|\det\!\left(\frac{\partial e_t}{\partial x_t}\right)\right|.
\]
Finally, Theorem~\ref{thm:jacobian_dispersion} provides the diagonal block formula
$\frac{\partial e_t}{\partial x_t}=I_d-\Sigma_tA$.
Therefore,
\[
\log|\det J(x)|
=
\sum_{t=1}^L \log\left|\det\!\left(I_d-\Sigma_tA\right)\right|.
\]
Combining terms and collecting constants yields \eqref{eq:loglik_matrix_decomp}:
\[
\log p(x_{1:L})
=
-\frac{1}{2\sigma^2}\sum_{t=1}^L\|e_t(x)\|_2^2
+
\sum_{t=1}^L\log\left|\det(I_d-\Sigma_tA)\right|
-
Ld\log\sigma
+\mathrm{const}.
\]
(The $-Ld\log\sigma$ term is included by splitting $-\frac{Ld}{2}\log(2\pi\sigma^2)$ into $-Ld\log\sigma$
plus an additive constant.)
\end{proof}

\subsection{Optional: spectral stability condition}
\label{app:spectral_condition}

\begin{lemma}[A sufficient condition for local invertibility]
\label{lem:spectral_sufficient}
If $\rho(\Sigma_tA)<1$ for all $t$ (spectral radius), then $I_d-\Sigma_tA$ is invertible for all $t$,
and hence the residual map $x\mapsto e$ is locally invertible.
\end{lemma}

\begin{proof}
If $\rho(M)<1$, then $1$ is not an eigenvalue of $M$, so $I-M$ has no zero eigenvalues and is invertible.
Apply with $M=\Sigma_tA$ for each $t$ and use Lemma~\ref{lem:block_lower_triangular}.
\end{proof}

\begin{remark}[Necessity vs.\ sufficiency]
The condition $\rho(\Sigma_tA)<1$ is sufficient, not necessary, for invertibility of $I_d-\Sigma_tA$.
The exact degeneracy boundary is $\det(I_d-\Sigma_tA)=0$, i.e., $1$ is an eigenvalue of $\Sigma_tA$.
When $\Sigma_tA$ is not symmetric, its eigenvalues may be complex; the condition $\rho(\Sigma_tA)<1$
guarantees $1$ is not an eigenvalue, but invertibility can still hold even if $\rho(\Sigma_tA)\ge 1$.
\end{remark}


\section{Measure-Theoretic Proofs: Consistency and Extension}
\label{app:kolmogorov_proofs}

\paragraph{Standing setup.}
Fix a finite vocabulary $V=\{v_1,\dots,v_K\}$ with discrete $\sigma$-algebra $2^V$.
For each $n\ge 1$, let $V^n$ be equipped with the product $\sigma$-algebra $2^{V^n}$.
Let $V^\infty=\prod_{t=1}^\infty V$ with the product $\sigma$-algebra $\mathcal F$ generated by cylinder sets.
We write $y_{1:n}=(y_1,\dots,y_n)\in V^n$.

For each $n\ge 1$, let $x_{1:n}\in\mathbb R^{dn}$ be latent embeddings and define the joint model
\begin{equation}
P_n(y_{1:n},x_{1:n})
\;=\;
P_n(y_{1:n}\mid x_{1:n})\,p_{\sigma,n}(x_{1:n}),
\label{eq:appendix_joint_n}
\end{equation}
with token marginal
\begin{equation}
P_n(y_{1:n})
\;=\;
\int_{\mathbb R^{dn}} P_n(y_{1:n}\mid x_{1:n})\,p_{\sigma,n}(x_{1:n})\,dx_{1:n}.
\label{eq:appendix_marginal_n}
\end{equation}

\paragraph{Assumptions (explicit).}
We will use the following assumptions, which match the causal construction in the main text.

\begin{assumption}[Causal token likelihood]
\label{ass:causal_token_lik}
For each $n$, the conditional distribution factorizes as
\begin{equation}
P_n(y_{1:n}\mid x_{1:n})
=
\prod_{t=1}^n P_n(y_t\mid y_{<t}, x_{\le t}),
\label{eq:appendix_obs_causal}
\end{equation}
and for $t\le n-1$ the factor $P_n(y_t\mid y_{<t},x_{\le t})$ does not depend on $n$ (only on the prefix).
\end{assumption}

\begin{assumption}[Causal embedding prior / projective structure]
\label{ass:causal_prior_projective}
For each $n\ge 2$, the embedding prior admits a sequential factorization
\begin{equation}
p_{\sigma,n}(x_{1:n})
=
p_{\sigma,n-1}(x_{1:n-1})\,p_{\sigma}(x_n\mid x_{<n}),
\label{eq:appendix_prior_factor}
\end{equation}
where the one-step conditional $p_{\sigma}(x_n\mid x_{<n})$ is a conditional density that integrates to $1$ for each fixed $x_{<n}$:
\begin{equation}
\int_{\mathbb R^{d}} p_{\sigma}(x_n\mid x_{<n})\,dx_n = 1.
\label{eq:appendix_prior_cond_normalized}
\end{equation}
Moreover, for every $t\le n-1$, the causal attention computation that defines the likelihood terms
and the prior conditionals at those times uses only $x_{\le t}$ (strict causality).
\end{assumption}

\begin{remark}
Assumption~\ref{ass:causal_token_lik} is the usual autoregressive decoding factorization.
Assumption~\ref{ass:causal_prior_projective} is the projective (prefix-stable) property implied by strict causality:
extending the length from $n-1$ to $n$ adds only the new conditional at time $n$ and does not alter
the earlier conditionals. In particular, strict-causal attention ensures $\mu_t$ for $t\le n-1$ depends
only on $x_{<t}$ and is unchanged when we append $(x_n,y_n)$.
\end{remark}


\begin{lemma}[Summing out the last token]
\label{lem:sum_out_last_token}
Under Assumption~\ref{ass:causal_token_lik}, for any fixed $n\ge 2$ and any fixed embeddings $x_{1:n}$,
\begin{equation}
\sum_{y_n\in V} P_n(y_{1:n}\mid x_{1:n})
=
P_{n-1}(y_{1:n-1}\mid x_{1:n-1}).
\label{eq:appendix_sum_out_last_token}
\end{equation}
\end{lemma}

\begin{proof}
Using the factorization \eqref{eq:appendix_obs_causal},
\[
P_n(y_{1:n}\mid x_{1:n})
=
\left(\prod_{t=1}^{n-1} P_n(y_t\mid y_{<t},x_{\le t})\right)\,P_n(y_n\mid y_{<n},x_{\le n}).
\]
Summing over $y_n$ and using that $P_n(\cdot\mid y_{<n},x_{\le n})$ is a categorical distribution,
\[
\sum_{y_n\in V} P_n(y_n\mid y_{<n},x_{\le n}) = 1.
\]
Thus,
\[
\sum_{y_n\in V} P_n(y_{1:n}\mid x_{1:n})
=
\prod_{t=1}^{n-1} P_n(y_t\mid y_{<t},x_{\le t}).
\]
By the prefix-stability clause in Assumption~\ref{ass:causal_token_lik}, for $t\le n-1$ these factors coincide with
$P_{n-1}(y_t\mid y_{<t},x_{\le t})$, hence the product equals $P_{n-1}(y_{1:n-1}\mid x_{1:n-1})$.
\end{proof}


\begin{lemma}[Integrating out the last embedding]
\label{lem:integrate_out_last_latent}
Under Assumption~\ref{ass:causal_prior_projective}, for any $n\ge 2$ and any measurable function
$f:\mathbb R^{d(n-1)}\to\mathbb R$ for which the integrals exist,
\begin{equation}
\int_{\mathbb R^{dn}} f(x_{1:n-1})\,p_{\sigma,n}(x_{1:n})\,dx_{1:n}
=
\int_{\mathbb R^{d(n-1)}} f(x_{1:n-1})\,p_{\sigma,n-1}(x_{1:n-1})\,dx_{1:n-1}.
\label{eq:appendix_integrate_out_last_latent}
\end{equation}
\end{lemma}

\begin{proof}
By the factorization \eqref{eq:appendix_prior_factor},
\[
p_{\sigma,n}(x_{1:n})
=
p_{\sigma,n-1}(x_{1:n-1})\,p_\sigma(x_n\mid x_{<n}).
\]
Therefore,
\[
\int_{\mathbb R^{dn}} f(x_{1:n-1})\,p_{\sigma,n}(x_{1:n})\,dx_{1:n}
=
\int_{\mathbb R^{d(n-1)}}\!\!\left[
f(x_{1:n-1})\,p_{\sigma,n-1}(x_{1:n-1})
\int_{\mathbb R^d} p_\sigma(x_n\mid x_{<n})\,dx_n
\right]dx_{1:n-1}.
\]
Using \eqref{eq:appendix_prior_cond_normalized}, the inner integral equals $1$, giving the claim.
\end{proof}


\subsection{Proof of Theorem~\ref{thm:kolmogorov_consistency}}
\label{app:proof_kolmogorov_consistency}

\begin{theorem}
Assume Assumptions~\ref{ass:causal_token_lik} and \ref{ass:causal_prior_projective}.
Then for all $n\ge 2$ and all $y_{1:n-1}\in V^{n-1}$,
\[
\sum_{y_n\in V} P_n(y_{1:n})
=
P_{n-1}(y_{1:n-1}).
\]
\end{theorem}

\begin{proof}
Start from the definition \eqref{eq:appendix_marginal_n} and sum over $y_n$:
\begin{align*}
\sum_{y_n\in V} P_n(y_{1:n})
&=
\sum_{y_n\in V} \int_{\mathbb R^{dn}} P_n(y_{1:n}\mid x_{1:n})\,p_{\sigma,n}(x_{1:n})\,dx_{1:n} \\
&=
\int_{\mathbb R^{dn}} \left(\sum_{y_n\in V} P_n(y_{1:n}\mid x_{1:n})\right)\,p_{\sigma,n}(x_{1:n})\,dx_{1:n},
\end{align*}
where we exchanged sum and integral (finite sum, so always valid).
By Lemma~\ref{lem:sum_out_last_token},
\[
\sum_{y_n\in V} P_n(y_{1:n}\mid x_{1:n})
=
P_{n-1}(y_{1:n-1}\mid x_{1:n-1}).
\]
Thus
\[
\sum_{y_n\in V} P_n(y_{1:n})
=
\int_{\mathbb R^{dn}} P_{n-1}(y_{1:n-1}\mid x_{1:n-1})\,p_{\sigma,n}(x_{1:n})\,dx_{1:n}.
\]
Now apply Lemma~\ref{lem:integrate_out_last_latent} with
$f(x_{1:n-1}) = P_{n-1}(y_{1:n-1}\mid x_{1:n-1})$:
\[
\int_{\mathbb R^{dn}} P_{n-1}(y_{1:n-1}\mid x_{1:n-1})\,p_{\sigma,n}(x_{1:n})\,dx_{1:n}
=
\int_{\mathbb R^{d(n-1)}} P_{n-1}(y_{1:n-1}\mid x_{1:n-1})\,p_{\sigma,n-1}(x_{1:n-1})\,dx_{1:n-1}.
\]
The right-hand side is exactly $P_{n-1}(y_{1:n-1})$ by \eqref{eq:appendix_marginal_n} with $n-1$.
\end{proof}


\subsection{Proof of Proposition~\ref{prop:noncausal_breaks_consistency}}
\label{app:proof_noncausal_breaks}

\begin{proposition}
If attention is non-causal so that earlier computations may depend on future positions (e.g.\ the attention
normalization at time $t$ depends on keys/values at positions $>t$), then in general the induced family
$\{P_n(y_{1:n})\}_{n\ge 1}$ need not satisfy Kolmogorov consistency:
$\sum_{y_n} P_n(y_{1:n})$ may differ from $P_{n-1}(y_{1:n-1})$.
\end{proposition}

\begin{proof}
We give an explicit counterexample with $n=2$ (failure already at the smallest nontrivial length).
Let $V=\{0,1\}$ and ignore latent embeddings (or take them deterministic) so the only issue is how the
token likelihood depends on sequence length through attention. Define a length-dependent (non-causal)
conditional at time $t=1$ by allowing it to depend on $y_2$:
\begin{equation}
P_2(y_1=1 \mid y_2=1)=0.9,\qquad P_2(y_1=1 \mid y_2=0)=0.1,
\label{eq:appendix_noncausal_dep}
\end{equation}
and let $P_2(y_2=1)=P_2(y_2=0)=1/2$. Define $P_1(y_1=1)=1/2$ for the length-$1$ system.
This is a minimal abstraction of non-causal attention: the distribution of the first position changes when
a future position is present.

Compute the induced length-$2$ marginal for $y_1$ by summing out $y_2$:
\[
\sum_{y_2\in\{0,1\}} P_2(y_1=1,y_2)
=
\sum_{y_2} P_2(y_1=1\mid y_2)\,P_2(y_2)
=
\frac{1}{2}\cdot 0.9 + \frac{1}{2}\cdot 0.1
=
0.5.
\]
In this particular choice it happens to match $P_1(y_1=1)=0.5$; now change the parameters slightly:
take $P_2(y_2=1)=0.9$ and $P_2(y_2=0)=0.1$ while keeping \eqref{eq:appendix_noncausal_dep}.
Then
\[
\sum_{y_2} P_2(y_1=1,y_2)
=
0.9\cdot 0.9 + 0.1\cdot 0.1
=
0.82,
\]
which differs from $P_1(y_1=1)=0.5$.
Thus the family fails the consistency requirement.

This phenomenon is exactly what non-causal attention can induce: earlier-position conditionals depend on
the presence/content of later positions (e.g.\ through unmasked softmax normalization), so the $(n-1)$-prefix
marginal extracted from the length-$n$ system need not match the from-scratch length-$(n-1)$ system.
\end{proof}

\begin{remark}
The proof above uses a length-dependent conditional as a stand-in for any mechanism by which adding a future token
changes the distribution of earlier tokens. In a non-causal self-attention block, the representation at position $t$
can depend on keys/values at positions $>t$, and therefore the categorical distribution at $t$ can change when the
sequence is extended. This violates the prefix-stability clause required in Lemma~\ref{lem:sum_out_last_token}.
\end{remark}


\subsection{Proof of Theorem~\ref{thm:kolmogorov_extension}}
\label{app:proof_kolmogorov_extension}

\begin{theorem}
Under the assumptions of Theorem~\ref{thm:kolmogorov_consistency}, there exists a unique probability
measure $P$ on $(V^\infty,\mathcal F)$ whose finite-dimensional marginals are $\{P_n\}_{n\ge 1}$:
for each $n$ and each $y_{1:n}\in V^n$,
\[
P(\{Y_{1:n}=y_{1:n}\}) = P_n(y_{1:n}).
\]
\end{theorem}

\begin{proof}
We verify the hypotheses of the Kolmogorov extension theorem for the projective family $\{P_n\}_{n\ge 1}$.

\textbf{Step 1.}
For each $n$ and each $y_{1:n}\in V^n$, define the cylinder set
\[
C(y_{1:n}) \;\triangleq\; \{ \omega\in V^\infty : \omega_1=y_1,\dots,\omega_n=y_n\}.
\]
Define $\nu$ on cylinders by $\nu(C(y_{1:n})) \triangleq P_n(y_{1:n})$.

\textbf{Step 2.}
If a cylinder is described at two different lengths, consistency ensures the same value.
Indeed, for any fixed $y_{1:n}$,
\[
\nu(C(y_{1:n}))
=
P_n(y_{1:n})
=
\sum_{y_{n+1}\in V} P_{n+1}(y_{1:n},y_{n+1})
=
\sum_{y_{n+1}} \nu(C(y_{1:n},y_{n+1})),
\]
so the mass assigned to a cylinder equals the sum of masses of its disjoint refinements.

\textbf{Step 3.}
The set of finite disjoint unions of cylinders forms an algebra that generates $\mathcal F$.
Using Step 2 (refinement) and the fact that each $P_n$ is a probability measure on the finite set $V^n$,
$\nu$ is finitely additive on this algebra.

\textbf{Step 4.}
Because $V$ is finite, the cylinder algebra is countable, and $\nu$ is automatically $\sigma$-additive
on it (one can check directly by refining to a common length and using $\sigma$-additivity on $V^n$).
Therefore $\nu$ is a pre-measure on a generating algebra of $\mathcal F$.
By Carath\'eodory's extension theorem, $\nu$ extends to a measure $P$ on $(V^\infty,\mathcal F)$.

\textbf{Step 5.}
Uniqueness follows because cylinder sets generate $\mathcal F$ and any two measures agreeing on a generating
$\pi$-system agree on the generated $\sigma$-algebra.

This proves existence and uniqueness of $P$ with the prescribed marginals.
\end{proof}

\begin{remark}
The only requirements are (i) that each $P_n$ is a probability measure on $V^n$ and (ii) the projective consistency relation
\eqref{eq:kolmogorov_consistency}. The role of the latent embedding construction is to \emph{justify} (ii)
for causal transformers via Theorem~\ref{thm:kolmogorov_consistency}.
\end{remark}


\section{Proofs for Section~\ref{sec:deep_hierarchy}}
\label{app:proofs_deep}

\paragraph{Notation.}
Fix a sequence length $T$ and number of layers $L_{\mathrm{layers}}$.
For each layer $\ell\in\{1,\dots,L_{\mathrm{layers}}\}$ and time $t\in\{1,\dots,T\}$, recall the
layerwise residual map
\begin{equation}
e^{(\ell)}_t\!\big(x^{(\ell)},x^{(\ell-1)}\big)
\;\triangleq\;
x^{(\ell)}_t - g^{(\ell)}_t\!\big(x^{(\ell-1)}_{1:t}\big),
\qquad
\varepsilon^{(\ell)}_t = e^{(\ell)}_t.
\label{eq:app_deep_residual}
\end{equation}
Let $x$ denote the concatenation of all layerwise embeddings
$\{x^{(\ell)}_{1:T}\}_{\ell=1}^{L_{\mathrm{layers}}}$, and similarly let
$\varepsilon$ denote the concatenation of all $\{\varepsilon^{(\ell)}_{1:T}\}_{\ell=1}^{L_{\mathrm{layers}}}$.
Define the global residual map
\[
E(x; x^{(0)}) \;=\; \varepsilon,
\]
obtained by stacking \eqref{eq:app_deep_residual} over all $(\ell,t)$.

\subsection{Triangular structure across depth and time}
\label{app:proof_deep_triangular}

\begin{lemma}[Block lower-triangular Jacobian for the deep residual map]
\label{lem:deep_block_triangular}
Order the variables $x^{(\ell)}_t$ lexicographically by $(\ell,t)$ (e.g., increasing $\ell$ and, within each layer, increasing $t$),
and stack $E$ in the same order.
Then the Jacobian $J_{x\mapsto\varepsilon}=\partial E/\partial x$ is block lower-triangular, with diagonal blocks
$\partial e^{(\ell)}_t/\partial x^{(\ell)}_t$.
Consequently,
\begin{equation}
\det J_{x\mapsto\varepsilon}(x)
\;=\;
\prod_{\ell=1}^{L_{\mathrm{layers}}}
\prod_{t=1}^{T}
\det\!\left(\frac{\partial e^{(\ell)}_t}{\partial x^{(\ell)}_t}\right).
\label{eq:app_deep_det_product}
\end{equation}
\end{lemma}

\begin{proof}
Fix $(\ell,t)$. By definition,
\[
e^{(\ell)}_t = x^{(\ell)}_t - g^{(\ell)}_t(x^{(\ell-1)}_{1:t}).
\]
Thus $e^{(\ell)}_t$ depends on $x^{(\ell)}_t$ directly, and on the previous-layer prefix $x^{(\ell-1)}_{1:t}$,
but it does \emph{not} depend on any $x^{(\ell')}_{t'}$ with $\ell'>\ell$, nor on any $x^{(\ell)}_{t'}$ with $t'\ne t$
(because the $x^{(\ell)}$-dependence is only through the explicit $x^{(\ell)}_t$ term).
Therefore, under the lexicographic ordering, all partial derivatives with respect to ``future'' variables
are zero and the Jacobian is block lower-triangular.
The determinant of a block lower-triangular matrix is the product of determinants of its diagonal blocks,
giving \eqref{eq:app_deep_det_product}.
\end{proof}

\begin{corollary}[Deep log-likelihood decomposition]
\label{cor:deep_loglik_decomp}
Assume the base noises are independent Gaussians
$\varepsilon^{(\ell)}_t\sim\mathcal{N}(0,\sigma_\ell^2 I_d)$.
Whenever $E$ is locally invertible at $x$, the induced log-density satisfies
\begin{align}
\log p(x\mid x^{(0)})
&=
\sum_{\ell=1}^{L_{\mathrm{layers}}}\sum_{t=1}^T
\log \mathcal{N}\!\left(\varepsilon^{(\ell)}_t; 0,\sigma_\ell^2 I_d\right)
+
\log\left|\det J_{x\mapsto\varepsilon}(x)\right| \nonumber\\
&=
-\sum_{\ell,t}\frac{1}{2\sigma_\ell^2}\left\|e^{(\ell)}_t(x)\right\|_2^2
-\sum_{\ell,t} d\log\sigma_\ell
+
\sum_{\ell,t}\log\left|\det\!\left(\frac{\partial e^{(\ell)}_t}{\partial x^{(\ell)}_t}\right)\right|
+\mathrm{const},
\label{eq:app_deep_loglik}
\end{align}
which is the expanded form of \eqref{eq:deep_logdet_decomp}.
\end{corollary}

\begin{proof}
Apply the change-of-variables formula:
$\log p(x)=\log p(\varepsilon)+\log|\det(\partial \varepsilon/\partial x)|$
with $\varepsilon=E(x)$.
The Gaussian base density gives the quadratic energy and $-\sum_{\ell,t}d\log\sigma_\ell$ term.
Lemma~\ref{lem:deep_block_triangular} gives the determinant factorization.
\end{proof}

\subsection{Proof of Proposition~\ref{prop:no_logdet_deep}}
\label{app:proof_no_logdet_deep}

\begin{proof}
Fix any layer $\ell\ge 2$ and time $t$.
By assumption, $g^{(\ell)}_t$ depends only on $x^{(\ell-1)}_{1:t}$ and does not depend on $x^{(\ell)}$.
Therefore the residual map at that layer/time is
\[
e^{(\ell)}_t(x^{(\ell)},x^{(\ell-1)}) = x^{(\ell)}_t - g^{(\ell)}_t(x^{(\ell-1)}_{1:t}),
\]
which is affine in $x^{(\ell)}_t$ with coefficient matrix $I_d$.
Hence
\[
\frac{\partial e^{(\ell)}_t}{\partial x^{(\ell)}_t} = I_d,
\]
and $\log\left|\det(\partial e^{(\ell)}_t/\partial x^{(\ell)}_t)\right|=\log|\det(I_d)|=0$.
\end{proof}

\subsection{Proof of Corollary~\ref{cor:localize_first}}
\label{app:proof_localize_first}

\begin{proof}
By Corollary~\ref{cor:deep_loglik_decomp}, the total stability/log-determinant contribution is
\[
\sum_{\ell=1}^{L_{\mathrm{layers}}}\sum_{t=1}^T
\log\left|\det\!\left(\frac{\partial e^{(\ell)}_t}{\partial x^{(\ell)}_t}\right)\right|.
\]
By Proposition~\ref{prop:no_logdet_deep}, every term with $\ell\ge 2$ is zero under the stated conditioning assumption.
Thus, the entire sum reduces to the $\ell=1$ contribution, which is exactly the single-layer stability term
derived in Section~\ref{sec:single_layer_margin} (scalar: $\log|1-a\,\mathrm{Var}_t|$; vector: $\log|\det(I_d-\Sigma_tA)|$).
\end{proof}

\subsection{Compatibility with multi-head attention, residuals, and normalization}
\label{app:proof_mha_residual_ln}

\begin{lemma}[Standard transformer blocks satisfy the conditioning assumption]
\label{lem:block_satisfies_conditioning}
Consider a standard pre-norm transformer block in which the attention weights in layer $\ell$ are computed
from $x^{(\ell-1)}$ (e.g., from $\mathrm{LN}(x^{(\ell-1)})$), and the layer noise enters additively at the
output:
\[
x^{(\ell)}_t = \tilde{x}^{(\ell)}_t + \varepsilon^{(\ell)}_t,
\qquad
\tilde{x}^{(\ell)}_t = g^{(\ell)}_t(x^{(\ell-1)}_{1:t}).
\]
Then $g^{(\ell)}_t$ does not depend on $x^{(\ell)}$, so Proposition~\ref{prop:no_logdet_deep} applies.
\end{lemma}

\begin{proof}
By construction, $\tilde{x}^{(\ell)}_t$ is computed entirely from previous-layer quantities
($x^{(\ell-1)}$ through LN/MHA/MLP/residual composition). The additive-noise step then defines
$x^{(\ell)}_t$ by adding $\varepsilon^{(\ell)}_t$ after $\tilde{x}^{(\ell)}_t$ is fixed.
Thus $g^{(\ell)}_t$ is a deterministic function of $x^{(\ell-1)}_{1:t}$ and does not depend on $x^{(\ell)}$.
\end{proof}

\begin{remark} If, instead, a layer is defined implicitly so that its attention weights depend on $x^{(\ell)}$ (the same layer variable being generated), then the residual map is no longer affine in $x^{(\ell)}_t$.
In that case $\partial e^{(\ell)}_t/\partial x^{(\ell)}_t$ need not be $I_d$ and an additional log-determinant term may appear, with an analogous margin-to-degeneracy interpretation.
\end{remark}

\section{Experimental Results Continued}
\label{app:additional_experiments}

\subsection{Perturbation robustness}
\label{app:appendix_robustness}

This appendix records the full robustness results summarized in
Section~\ref{sec:exp_robustness}. The main text shows the
\texttt{wikitext-103} robustness figure and a compact two-dataset endpoint
table. Here we provide the corresponding \texttt{simplebooks-92} figure and
the full numeric results for both datasets. As in the main text, we report
results for two embedding-space perturbation families: unstructured Gaussian
noise and structured Gaussian drift.

\begin{figure}[H]
\centering
\includegraphics[width=\linewidth]{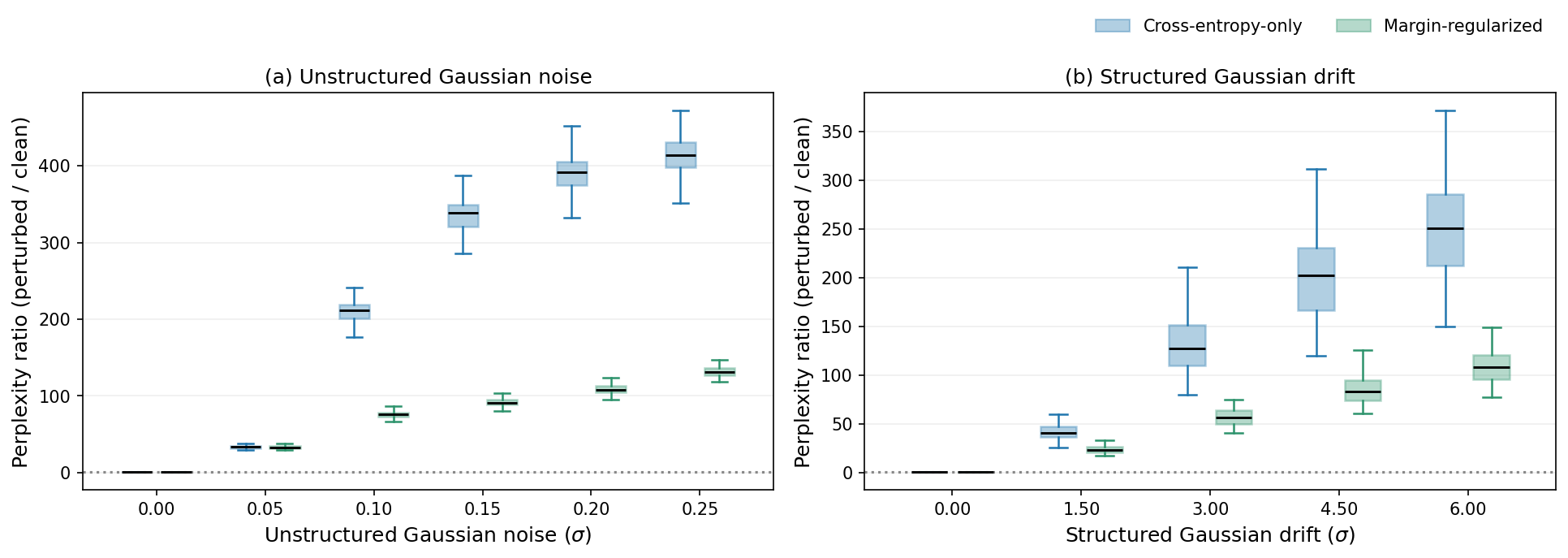}
\caption{\textbf{Robustness on \texttt{simplebooks-92}: relative degradation under unstructured Gaussian noise and structured Gaussian drift.}
(a)~Unstructured Gaussian noise: grouped boxplots of perturbed perplexity / clean perplexity across 100 validation subsamples for $\sigma \in \{0, 0.05, 0.10, 0.15, 0.20, 0.25\}$. (b)~Structured Gaussian drift: grouped boxplots for $\sigma \in \{0, 1.5, 3.0, 4.5, 6.0\}$. The same qualitative pattern as in the main text appears: a strong advantage for margin regularization under both perturbation families.}
\label{fig:appendix_robustness_simplebooks}
\end{figure}

\begin{table}[H]
\centering
\caption{Unstructured Gaussian noise robustness on \texttt{wikitext-103}.
Entries report the median perplexity ratio (perturbed perplexity / clean perplexity); each estimate is followed by its empirical 95\% interval on the line below.}
\label{tab:appendix_robustness_wikitext_gaussian}
\small
\begin{tabular}{lcc}
\toprule
\textbf{Level} & \textbf{Cross-entropy-only} & \textbf{Margin-regularized} \\
\midrule
Clean & 1.00 & 1.00 \\
$\sigma{=}0.05$ & 51.10 & 55.75 \\
\multicolumn{1}{r}{\footnotesize [95\%]} & \footnotesize [44.71, 62.41] & \footnotesize [48.01, 70.16] \\
$\sigma{=}0.10$ & 311.32 & 116.99 \\
\multicolumn{1}{r}{\footnotesize [95\%]} & \footnotesize [271.24, 401.72] & \footnotesize [99.58, 152.81] \\
$\sigma{=}0.15$ & 525.27 & 180.72 \\
\multicolumn{1}{r}{\footnotesize [95\%]} & \footnotesize [452.52, 665.51] & \footnotesize [152.88, 236.66] \\
$\sigma{=}0.20$ & 681.58 & 269.18 \\
\multicolumn{1}{r}{\footnotesize [95\%]} & \footnotesize [579.39, 854.15] & \footnotesize [228.48, 351.90] \\
$\sigma{=}0.25$ & 775.81 & 364.92 \\
\multicolumn{1}{r}{\footnotesize [95\%]} & \footnotesize [661.68, 974.21] & \footnotesize [308.30, 473.11] \\
\bottomrule
\end{tabular}
\end{table}

\begin{table}[H]
\centering
\caption{Structured Gaussian drift robustness on \texttt{wikitext-103}.
Entries report the median perplexity ratio (perturbed perplexity / clean perplexity); each estimate is followed by its empirical 95\% interval on the line below.}
\label{tab:appendix_robustness_wikitext_structured}
\small
\begin{tabular}{lcc}
\toprule
\textbf{Level} & \textbf{Cross-entropy-only} & \textbf{Margin-regularized} \\
\midrule
Clean & 1.00 & 1.00 \\
$\sigma{=}1.50$ & 59.48 & 50.88 \\
\multicolumn{1}{r}{\footnotesize [95\%]} & \footnotesize [38.07, 86.18] & \footnotesize [35.12, 83.29] \\
$\sigma{=}3.00$ & 201.56 & 135.75 \\
\multicolumn{1}{r}{\footnotesize [95\%]} & \footnotesize [131.63, 312.16] & \footnotesize [86.22, 219.54] \\
$\sigma{=}4.50$ & 329.51 & 207.99 \\
\multicolumn{1}{r}{\footnotesize [95\%]} & \footnotesize [226.48, 509.38] & \footnotesize [133.55, 310.42] \\
$\sigma{=}6.00$ & 419.11 & 273.24 \\
\multicolumn{1}{r}{\footnotesize [95\%]} & \footnotesize [304.50, 674.97] & \footnotesize [180.95, 390.50] \\
\bottomrule
\end{tabular}
\end{table}

\begin{table}[H]
\centering
\caption{Unstructured Gaussian noise robustness on \texttt{simplebooks-92}.
Entries report the median perplexity ratio (perturbed perplexity / clean perplexity); each estimate is followed by its empirical 95\% interval on the line below.}
\label{tab:appendix_robustness_simplebooks_gaussian}
\small
\begin{tabular}{lcc}
\toprule
\textbf{Level} & \textbf{Cross-entropy-only} & \textbf{Margin-regularized} \\
\midrule
Clean & 1.00 & 1.00 \\
$\sigma{=}0.05$ & 33.41 & 32.82 \\
\multicolumn{1}{r}{\footnotesize [95\%]} & \footnotesize [29.32, 36.88] & \footnotesize [29.71, 37.11] \\
$\sigma{=}0.10$ & 212.11 & 75.98 \\
\multicolumn{1}{r}{\footnotesize [95\%]} & \footnotesize [187.51, 233.99] & \footnotesize [67.99, 85.44] \\
$\sigma{=}0.15$ & 338.38 & 90.81 \\
\multicolumn{1}{r}{\footnotesize [95\%]} & \footnotesize [301.94, 379.18] & \footnotesize [82.11, 103.22] \\
$\sigma{=}0.20$ & 391.16 & 108.22 \\
\multicolumn{1}{r}{\footnotesize [95\%]} & \footnotesize [353.46, 440.16] & \footnotesize [98.57, 123.29] \\
$\sigma{=}0.25$ & 413.45 & 130.91 \\
\multicolumn{1}{r}{\footnotesize [95\%]} & \footnotesize [372.86, 462.31] & \footnotesize [118.78, 148.52] \\
\bottomrule
\end{tabular}
\end{table}

\begin{table}[H]
\centering
\caption{Structured Gaussian drift robustness on \texttt{simplebooks-92}.
Entries report the median perplexity ratio (perturbed perplexity / clean perplexity); each estimate is followed by its empirical 95\% interval on the line below.}
\label{tab:appendix_robustness_simplebooks_structured}
\small
\begin{tabular}{lcc}
\toprule
\textbf{Level} & \textbf{Cross-entropy-only} & \textbf{Margin-regularized} \\
\midrule
Clean & 1.00 & 1.00 \\
$\sigma{=}1.50$ & 41.05 & 23.25 \\
\multicolumn{1}{r}{\footnotesize [95\%]} & \footnotesize [31.08, 58.49] & \footnotesize [18.62, 30.14] \\
$\sigma{=}3.00$ & 127.55 & 56.74 \\
\multicolumn{1}{r}{\footnotesize [95\%]} & \footnotesize [84.59, 214.33] & \footnotesize [44.73, 74.50] \\
$\sigma{=}4.50$ & 202.22 & 83.00 \\
\multicolumn{1}{r}{\footnotesize [95\%]} & \footnotesize [124.10, 357.32] & \footnotesize [63.80, 119.43] \\
$\sigma{=}6.00$ & 250.71 & 107.96 \\
\multicolumn{1}{r}{\footnotesize [95\%]} & \footnotesize [158.38, 455.65] & \footnotesize [83.02, 155.03] \\
\bottomrule
\end{tabular}
\end{table}

\subsection{Choosing the margin penalty}
\label{app:appendix_lambda}

This appendix records the full $\lambda_m$ results summarized in
Section~\ref{sec:exp_lambda}. The main text shows the
\texttt{wikitext-103} figure and a compact two-dataset comparison over the
three most relevant settings. Here we add the corresponding
\texttt{simplebooks-92} figure and provide the full two-dataset results for all
five trained values,
\[
\lambda_m \in \{0.00, 0.02, 0.05, 0.10, 0.25\},
\]
including both predictive/robustness metrics and the two geometric metrics (contextual embedding error and effective support size).
As in the main text, effective support size is reported only when a trained prior exists.
Thus the $\lambda_m{=}0.00$ column omits that metric; the corresponding no-prior geometry is given by the random-initialized baseline in Section~\ref{sec:exp_support_tokens}.

\begin{figure}[H]
\centering
\includegraphics[width=\linewidth]{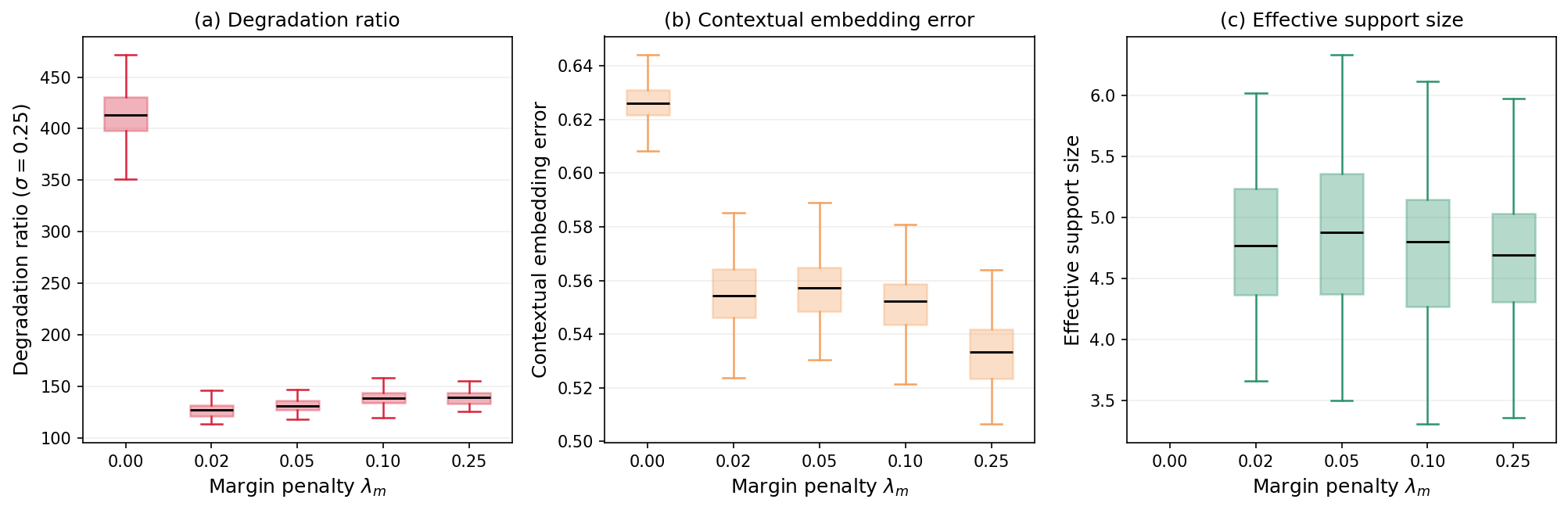}
\caption{\textbf{Selecting $\lambda_m$ on \texttt{simplebooks-92}.}
Grouped boxplots over 100 validation subsamples for
$\lambda_m \in \{0.00, 0.02, 0.05, 0.10, 0.25\}$.
(a)~Degradation ratio at fixed unstructured Gaussian noise $\sigma{=}0.25$.
(b)~Contextual embedding error.
(c)~Effective support size.}
\label{fig:appendix_lambda_simplebooks}
\end{figure}

\begin{table}[H]
\centering
\caption{Regularization metrics on \texttt{wikitext-103}.
Degradation ratio is the perturbed perplexity / clean perplexity at fixed unstructured Gaussian noise $\sigma{=}0.25$.
Contextual embedding error is the normalized squared distance between the true token embedding and the embedding predicted from the model's next-token distribution given the preceding context.
Effective support size is the entropy-based effective number of support tokens under the learned prior.
It is not reported at $\lambda_m{=}0.00$ because no prior is trained
there.
Each estimate is followed by its empirical 95\% interval on the line below.}
\label{tab:appendix_lambda_wikitext}
\small
\begin{tabular}{lcccc}
\toprule
\shortstack[c]{\textbf{$\lambda_m$}\\\mbox{}} & \shortstack[c]{\strut\textbf{Clean}\\\textbf{perplexity}\strut} & \shortstack[c]{\strut\textbf{Degradation}\\\textbf{ratio}\strut} & \shortstack[c]{\strut\textbf{Contextual}\\\textbf{embedding error}\strut} & \shortstack[c]{\strut\textbf{Effective}\\\textbf{support size}\strut} \\
\midrule
0.00 & 15.68 & 775.81 & 0.608 & --- \\
\multicolumn{1}{r}{\footnotesize [95\%]} & \footnotesize [12.40, 18.10] & \footnotesize [661.68, 974.21] & \footnotesize [0.567, 0.630] & \footnotesize --- \\
0.02 & 15.74 & 328.51 & 0.583 & 2.75 \\
\multicolumn{1}{r}{\footnotesize [95\%]} & \footnotesize [12.53, 18.15] & \footnotesize [273.81, 424.92] & \footnotesize [0.550, 0.614] & \footnotesize [2.06, 3.77] \\
0.05 & 15.71 & 364.92 & 0.596 & 3.10 \\
\multicolumn{1}{r}{\footnotesize [95\%]} & \footnotesize [12.50, 18.12] & \footnotesize [308.30, 473.11] & \footnotesize [0.563, 0.625] & \footnotesize [2.13, 4.02] \\
0.10 & 15.81 & 381.28 & 0.581 & 4.43 \\
\multicolumn{1}{r}{\footnotesize [95\%]} & \footnotesize [12.78, 18.27] & \footnotesize [320.68, 473.21] & \footnotesize [0.551, 0.609] & \footnotesize [3.37, 5.59] \\
0.25 & 15.83 & 393.06 & 0.583 & 3.48 \\
\multicolumn{1}{r}{\footnotesize [95\%]} & \footnotesize [12.82, 18.14] & \footnotesize [337.46, 491.80] & \footnotesize [0.552, 0.612] & \footnotesize [2.50, 4.80] \\
\bottomrule
\end{tabular}
\end{table}

\begin{table}[H]
\centering
\caption{Regularization metrics on \texttt{simplebooks-92}. Degradation
ratio, contextual embedding error, and effective support size are defined as
in Table~\ref{tab:appendix_lambda_wikitext}. Each estimate is followed by its
empirical 95\% interval on the line below.}
\label{tab:appendix_lambda_simplebooks}
\small
\begin{tabular}{lcccc}
\toprule
\shortstack[c]{\textbf{$\lambda_m$}\\\mbox{}} & \shortstack[c]{\strut\textbf{Clean}\\\textbf{perplexity}\strut} & \shortstack[c]{\strut\textbf{Degradation}\\\textbf{ratio}\strut} & \shortstack[c]{\strut\textbf{Contextual}\\\textbf{embedding error}\strut} & \shortstack[c]{\strut\textbf{Effective}\\\textbf{support size}\strut} \\
\midrule
0.00 & 10.67 & 413.45 & 0.626 & --- \\
\multicolumn{1}{r}{\footnotesize [95\%]} & \footnotesize [9.76, 11.76] & \footnotesize [372.86, 462.31] & \footnotesize [0.610, 0.641] & \footnotesize --- \\
0.02 & 10.75 & 127.30 & 0.554 & 4.77 \\
\multicolumn{1}{r}{\footnotesize [95\%]} & \footnotesize [9.82, 11.71] & \footnotesize [114.70, 142.45] & \footnotesize [0.532, 0.580] & \footnotesize [3.83, 5.85] \\
0.05 & 10.67 & 130.91 & 0.557 & 4.88 \\
\multicolumn{1}{r}{\footnotesize [95\%]} & \footnotesize [9.81, 11.67] & \footnotesize [118.78, 148.52] & \footnotesize [0.534, 0.586] & \footnotesize [3.83, 6.15] \\
0.10 & 10.74 & 138.09 & 0.552 & 4.80 \\
\multicolumn{1}{r}{\footnotesize [95\%]} & \footnotesize [9.79, 11.70] & \footnotesize [125.11, 153.33] & \footnotesize [0.530, 0.579] & \footnotesize [3.75, 5.77] \\
0.25 & 10.71 & 139.22 & 0.533 & 4.69 \\
\multicolumn{1}{r}{\footnotesize [95\%]} & \footnotesize [9.89, 11.73] & \footnotesize [127.26, 154.09] & \footnotesize [0.508, 0.562] & \footnotesize [3.81, 5.81] \\
\bottomrule
\end{tabular}
\end{table}

\end{document}